\documentclass{article}

\PassOptionsToPackage{numbers, compress}{natbib}

\usepackage[final]{neurips_2023}
\usepackage{style}
\usepackage{tcolorbox}
\usepackage{subfigure}

\usepackage[utf8]{inputenc} 
\usepackage[T1]{fontenc}    
\usepackage{hyperref}       
\usepackage{url}            
\usepackage{booktabs}       
\usepackage{amsfonts}       
\usepackage{nicefrac}       
\usepackage{microtype}      
\usepackage{xcolor}         

\title{ResMem: Learn what you can and memorize the rest}

\author{%
  Zitong Yang \\
  Stanford University\\
  Stanford, CA 94305 \\
  \texttt{zitong@berkeley.edu} \\
  \And
  Michal Lukasik \\
  Google Research \\
  New York, NY, 10011 \\
  \texttt{mlukasik@google.com} \\
  \And
  Vaishnavh Nagarajan \\
  Google Research \\
  New York, NY, 10011 \\
  \texttt{vaishnavh@google.com} \\
  \And
  Zonglin Li \\
  Google Research \\
  New York, NY, 10011 \\
  \texttt{lizonglin@google.com} \\
  \And
  Ankit Singh Rawat \\
  Google Research \\
  New York, NY, 10011 \\
  \texttt{ankitsrawat@google.com} \\
  \And
  Manzil Zaheer \\
  Google Research \\
  New York, NY, 10011 \\
  \texttt{manzilzaheer@google.com} \\
  \And
  Aditya Krishna Menon \\
  Google Research \\
  New York, NY, 10011 \\
  \texttt{adityakmenon@google.com} \\
  \And
  Sanjiv Kumar \\
  Google Research \\
  New York, NY, 10011 \\
  \texttt{sanjivk@google.com} \\
}

\begin{document}

\maketitle

\begin{abstract}
    The impressive generalization performance of modern neural networks is attributed in part to 
their ability to \emph{implicitly} memorize complex training patterns.
Inspired by this, we explore a novel mechanism to improve model generalization via \emph{explicit} memorization.
Specifically, we propose the \emph{residual-memorization} (\emph{ResMem}) algorithm, a new method that augments an existing prediction model (e.g., a neural network) by fitting the model's residuals with a $k$-nearest neighbor based regressor.
The final prediction is then the sum of the original model and the fitted residual regressor.
By construction, ResMem can explicitly memorize the training labels, even when the base model has low capacity.
We start by formulating a stylized linear regression problem and rigorously show that ResMem results in a more favorable test risk over a base linear neural network.
Then, we empirically show that ResMem consistently improves the test set generalization of the original prediction model across standard vision and natural language processing benchmarks.
\end{abstract}

\section{Introduction}
Large neural networks achieve remarkable \emph{generalization} on test samples 
despite \emph{memorization} 
of training samples, in the sense of achieving zero training error~\citep{Zhang:2017}.
Several recent analyses have established that, under certain settings, memorization is \emph{sufficient} to achieve generalization~\citep{Bartlett:2017,
Dziugaite:2017,Belkin:2018,
Neyshabur:2019,Bartlett:2020},
and, more surprisingly, 
can even be \emph{necessary}~\citep{Feldman2020, FeldmanZhang2020, CJK2022Mem}.
These works suggest that suitable memorization can be a valuable desiderata for learning.
While increasing model size is a conceptually simple strategy to enable memorization, this has the obvious downside
of significantly increasing the cost of model training and serving.
This raises a natural question:
\emph{are there alternate mechanisms to improve the memorization (and thus generalization) of a relatively small model?}

\begin{figure*}[t]
    \centering
    \subfigure[\label{fig:cartoon1} \textbf{Step 1:} learn the training set.]{\includegraphics[width=.32\textwidth]{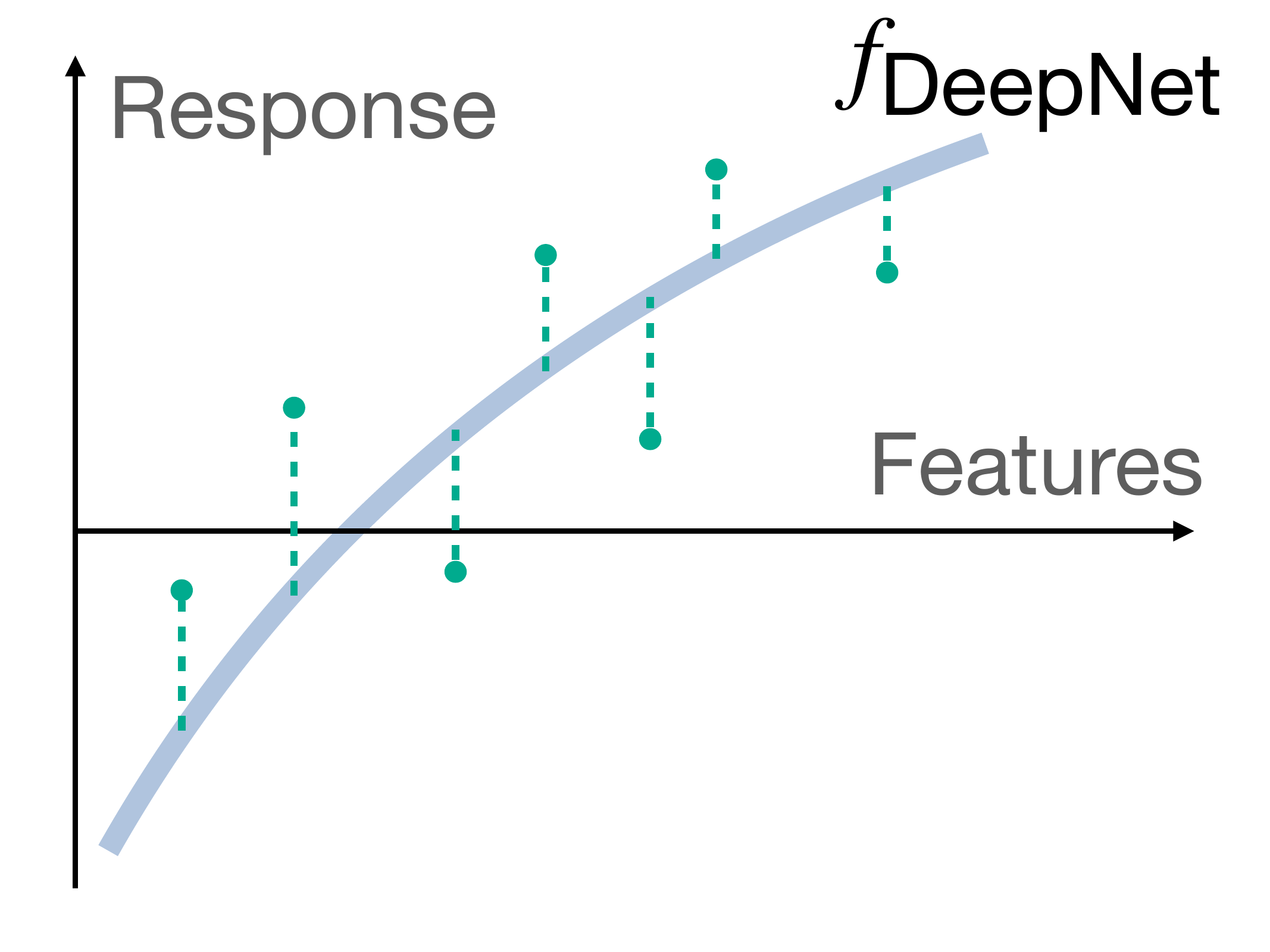}}
    \subfigure[\label{fig:cartoon2}\textbf{Step 2:} compute the residual.]{\includegraphics[width=.32\textwidth]{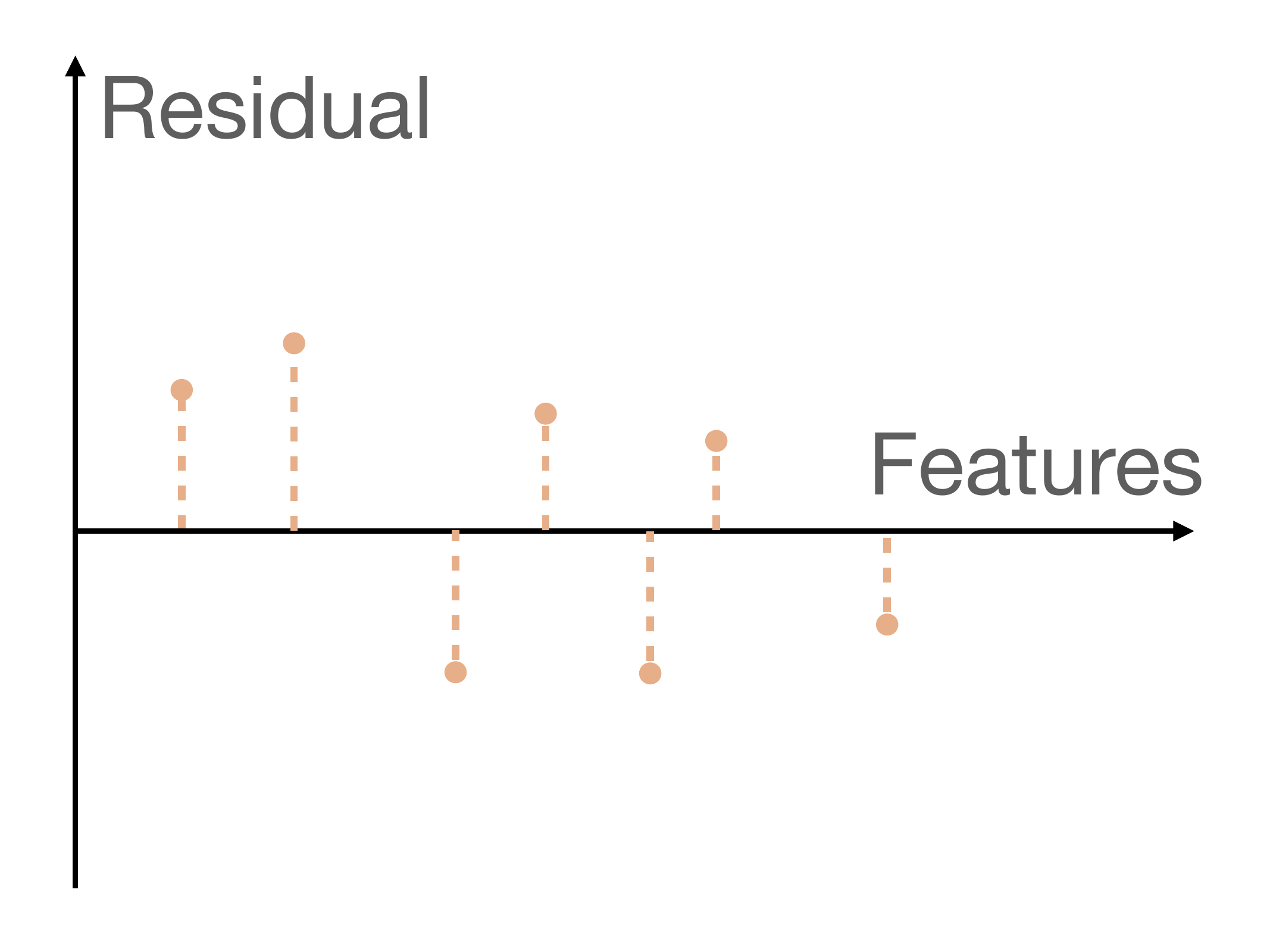}}
    \subfigure[\label{fig:cartoon3}\textbf{Step 3:} memorize the residual.]{\includegraphics[width=.32\textwidth]{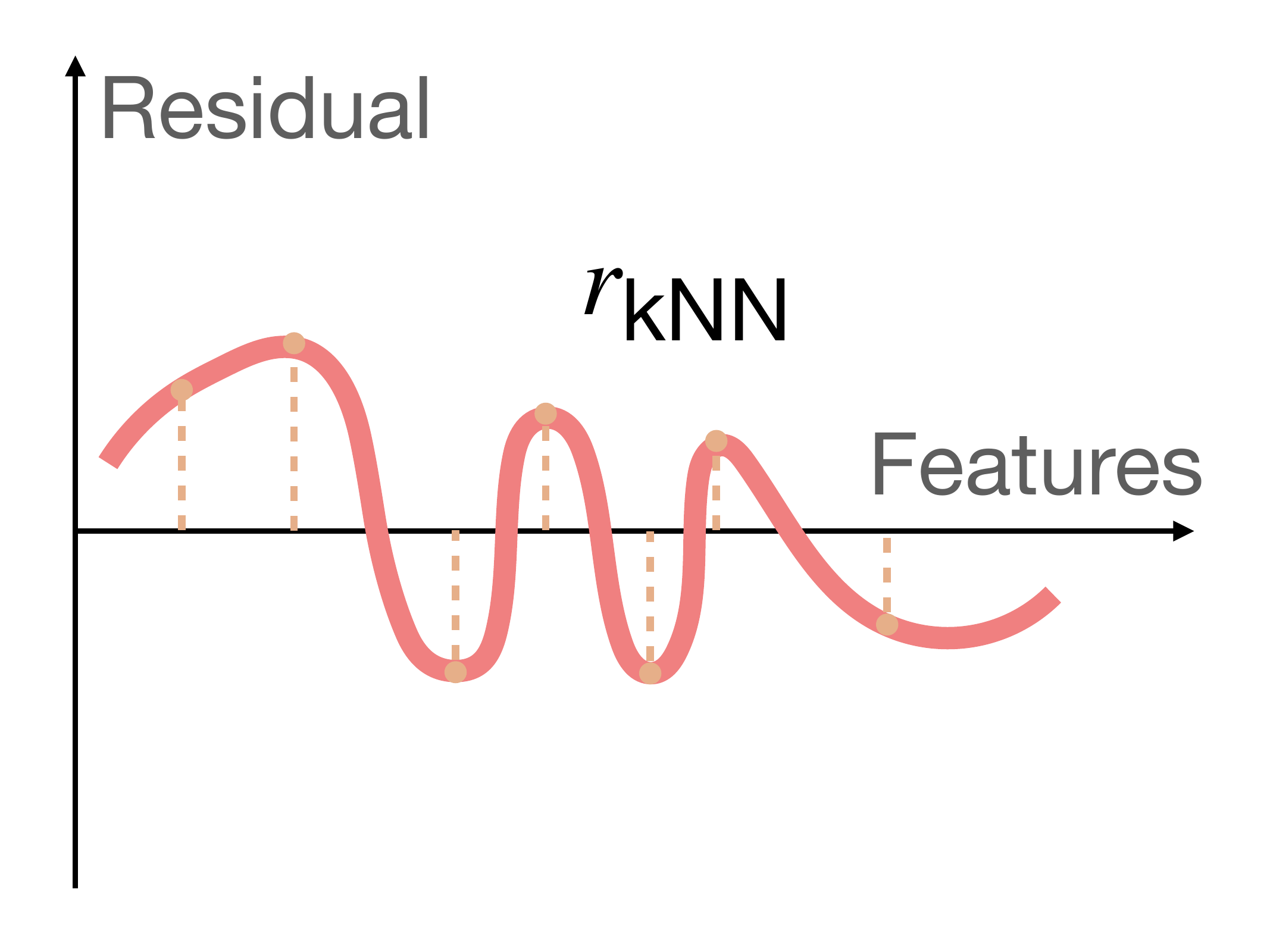}}
    \caption{Illustration of the \emph{residual memorization} (ResMem) algorithm.
    In a nutshell, we first fit a small deep network $f_{\sf DeepNet}$ on the training sample (Figure \ref{fig:cartoon1}).
    When this network is non-memorizing, it incurs non-zero \emph{residual} errors in its predictions (Figure \ref{fig:cartoon2}).
    We then fit a $k$-nearest neighbor based regressor on these residuals (Figure \ref{fig:cartoon3}).
    The final prediction is given by the sum of the initial network and $k$-NN regressor predictions.
    In all three figures, the $x$-axis represents the features in a supervised learning problem.
    In Figure \ref{fig:cartoon1}, the $y$-axis represents the targets of prediction.
    In Figure \ref{fig:cartoon2} and \ref{fig:cartoon3}, the $y$-axis represents the residual of the initial fitting from \textbf{Step 1}.}
    \label{fig:cartoon}
\end{figure*}

In this paper, we propose \emph{residual memorization} (\emph{ResMem}), a simple yet effective mechanism that achieves this goal (cf.~Figure \ref{fig:cartoon}).
Compared to the \emph{implicit} memorization performed by large neural models, the key idea behind ResMem is to perform \emph{explicit} memorization via a separate $k$-nearest neighbor component.
Specifically, ResMem involves first training a standard neural network $f_\nn$, 
and then  explicitly \emph{memorizing the model's residuals}
with a $k$-nearest neighbor based regressor $r_{\knn}$.
Memorization through $k$-nearest neighbor can be efficiently computed with various approximation schemes (e.g. \cite{avq_2020}).
Subsequently, the ResMem prediction on an instance $x$ is given by the sum of the two components, i.e., $f_\nn(x) + r_\knn(x)$.

We start by formulating a stylized linear regression problem that captures the essence behind ResMem~(cf.~Section \ref{sec:thy}).
Our analysis (Theorem \ref{thm:test-loss-ltm}) shows that, without ResMem, the test risk of the base linear neural network decreases to an irreducible constant as the number of samples goes to infinity.
In contrast, the test risk of ResMem decreases to zero.
The insight of theoretical analysis is that ResMem augments the capacity of the parametric linear network by adding a non-parametric component (i.e., nearest-neighbor).\\

Empirically, we show that such explicit memorization indeed leads to generalization benefits:
ResMem consistently improves the test accuracy of a baseline DeepNet on image classification tasks with CIFAR100~\citep{cifar}, and autoregressive language modeling on C4~\citep{2020t5} (Section \ref{sec:exp}).
Towards understanding this improved performance, we hypothesize that ResMem works by learning in two-stages~(cf.~Section \ref{sec:emp-ana}).
Specifically, we posit that the initial DeepNet $f_\nn$ learns some \emph{coarse} structure, and ResMem $r_\knn$ supplements 
the DeepNet prediction with \emph{fine-grained} details (cf. Figure \ref{fig:cifar-nlp-analysis}).
We verify our hypothesis via qualitative analysis on CIFAR100 and C4 (Section~\ref{sec:emp-ana}).

\noindent To summarize, our contributions are:
\begin{enumerate}[label=(\arabic*), leftmargin=16pt]
\item We propose \emph{residual-memorization} (\emph{ResMem}), a two-stage learning algorithm that combines a base prediction model with a nearest neighbor regressor (cf.~Figure \ref{fig:cartoon});
\item We theoretically analyze the rate of convergence of ResMem on a stylized linear regression problem, and show that it can improve upon the base prediction model (Section~\ref{sec:thy}).
\item We empirically demonstrate that ResMem improves test performance of neural networks (cf.~Section~\ref{sec:exp}), particularly when the training set is extremely large;
\end{enumerate}

\subsection{Applicable scenarios of ResMem}
\label{sec:applicable}

From our theoretical and empirical analysis, we posit that ResMem (Figure \ref{fig:cartoon}) yields the largest margin of improvement over a base DeepNet when it is infeasible to perform \emph{implicit} memorization with the latter.
We discuss three such scenarios below.
Each of our main empirical or theoretical results roughly corresponds to one of these settings.

\begin{itemize}[leftmargin=5mm, itemsep=1mm, partopsep=0pt,parsep=0pt]
    \item \textbf{Complex dataset.} In this scenario, the Bayes-optimal decision boundary is very complex, and is beyond the capability of the neural network itself. 
    To demonstrate this, we analyze a theoretical linear regression problem where the target regression function is not contained in the hypothesis class of linear neural networks (cf.~Section \ref{sec:thy}).

    \item \textbf{Large sample size.} Here, the number of training samples is large enough to make training set interpolation (i.e., achieving zero training error) infeasible for a given neural network model.
    For example, current large language models (LLMs) 
    may be trained
    for at most a single epoch over trillions of examples~\citep{Chowdhery:2022}.    
    By contrast, ResMem can circumvent this issue by explicitly memorizing the training samples. 
    We emulate this scenario by considering a causal language modeling task on the C4 dataset (cf.~Section \ref{sec:exp-nlp}).   

    \item \textbf{Small model.} In many practical settings, one may prefer a smaller model over a state-of-the-art model due to the training and deployment cost constraints.
    We emulate such a setting through an image classification task where it is indeed feasible to memorize the training data perfectly using state-of-the-art neural networks, but instead, we use smaller neural networks for computational efficiency (cf. Section \ref{sec:exp-vision}).
\end{itemize}

%
\section{Related work}
\label{sec:relatedwork}

We discuss two types of related work: Section \ref{sec:relatedwork-mem-gen} for literature on memorization and generalization that motivates the ResMem algorithm; Section \ref{sec:relatedwork-other-alg} for other related algorithms similar to ResMem.

\subsection{Memorization for generalization: prior work}
\label{sec:relatedwork-mem-gen}
\paragraph{Memorization is compitable for generalization.}
Overparameterized neural models with many more parameters than training samples have the capacity to perfectly fit (or \emph{interpolate}) even random training labels~\citep{Zhang:2017};
i.e., they can drive the empirical loss to zero for \emph{any} training set.
At the same time, when trained on real-world datasets, 
increasing model complexity tends to \emph{improve} model performance~\citep{Neyshabur:2019, pmlr-v119-yang20j};
that is, the models do not simply memorize the training sample, but rather learn generalizable patterns.
Several works have sought to understand the reasons behind this behaviour, both empirically~\citep{Arpit:2017} and theoretically~\citep{Bartlett:2017,Dziugaite:2017,Brutzkus:2018,Belkin:2018,Neyshabur:2019,Liang:2020,Montanari:2020,Bartlett:2020,Vapnik:2021,Wang:2021b, pmlr-v139-yang21a}.
One recurring message from the theory is that memorization (in the form of interpolation) can be sufficient for generalization.

\paragraph{Memorization can be necessary for generalization.}
Some recent works~\citep{Feldman2020,CJK2022Mem} showed that memorization
---
either in the sense of interpolation, or in a more general sense of stability~\citep{Feldman2020_DiscoveringLT}
---
may be \emph{necessary} for generalization.
\citet{Feldman:2019} considered a setting where the label distribution exhibits a \emph{long-tailed} distribution, and showed that to prevent incurring a large error on the large number of under-represented classes, it may be necessary to memorize many of their associated training samples.
\citet{CJK2022Mem} considered a linear regression setting where it is beneficial to fit the training targets to error lower than the Bayes-error (i.e., the inherent noise in the targets).

\subsection{Relation to existing algorithms}
\label{sec:relatedwork-other-alg}

\paragraph{Nearest neighbor method.} The $k$-nearest neighbor ($k$-NN)~\citep{NNCover, KnuthAOP, Hastie:2001, locallearning} method assigns label to a test sample based on the label of its nearest neighbor(s) in the training set.
Owing to its \emph{simplicity}, \emph{flexibility} in defining input similarity, and \emph{computational efficiency} with various approximation schemes~\citep{Gionis1999SimilaritySI, Muja2009FastAN}, this method remains popular.
However, the performance of $k$-NN drops as data becomes high dimensional~\citep{nnrates, Muja2009FastAN}.
Therefore, to apply it to high dimensional data such as image and text~\citep{knnsvm}, one approach is to learn a representation of data using neural networks~\citep{pmlr-v2-salakhutdinov07a}.
Following this approach, \citep{arXivDnnKnn} finds that applying $k$-NN directly to memorize the training labels $y_i$ yields similar performance with the original softmax based neural network classification.
In contrast, ResMem applies $k$-NN to memorize the \emph{residual} $r_i$ over the predictions of a base network.

%
\paragraph{Boosting and residual fitting.}
Boosting algorithms such as AdaBoost~\citep{Freund:1997} seek to construct an ensemble of ``weak learner'' models with good generalization.
AdaBoost achieves this in an iterative manner, and can be interpreted as a particular instantiation of \emph{forward stage-wise regression}~\citep{Friedman:2000}, a classical procedure from statistics~\citep{Goldberger:1961,Alley:1987,Tibshirani:2015}.
Intuitively, at each round, one builds a new weaker learner by fitting the residual of the ensemble of weak learners constructed thus far.
This fitting is performed iteratively.
ResMem can be loosely regarded as a two round boosting algorithm where the first ``weak learner'' is the base neural network and the second ``weak learner'' is the nearest-neighbor component.
Note that there is no need for the thrid ``weak learner'', because the nearest-neighbor component already perfectly memorizes the neural network residuals.

\paragraph{Memory-augmented language models.}
In the language modelling literature, several works explore combining neural models with an external database or memory, which can be queried to retrieve additional context~\citep{Lample:2019,Guu:2020,Borgeaud:2021,li2022decoupled}.
Closer to our work,~\citet{Khandelwal:2020} employ a linear combination of neural network and $k$-NN classifier components.
However, a crucial difference is that our $k$-NN components memorizes the \emph{residual} of the DeepNet prediction, whereas \citet{Khandelwal:2020} memorizes the \emph{target label} directly;
i.e., their approach is akin to an ensemble of $k$-NN and a deep network.
Various forms of memory have also been considered in generic classification problems~\citep{Panigrahy:2021,Vapnik:2021,Wang:2022}.
This line of literature again differs from ResMem in that their memory tries to memorize labels directly, whereas ResMem memorizes the \emph{residuals}, leading to a natural combination of the neural network and the memory component.

\paragraph{Model compression for large neural networks.}
Since ResMem boosts the test accuracy of a small, non-memorizing neural network, we can also view it as a technique that allows a small network to match the performance of a larger one.
This relates to the model compression literature.
Distillation~\citep{Hinton:2015, Bucilua:2006} is a popular strategy for compressing a large neural model to a smaller one.
For a survey of other effective strategies, including pruning, see~\citet{Menghani:2021}.
In Appendix \ref{sec:sensitivity-cifar}, we discuss how ResMem can be regarded as a ``dual procedure'' of distillation.

\section{Theoretical results}
\label{sec:thy}
As discussed in Section \ref{sec:applicable}, ResMem yields the largest improvement when implicit memorization is infeasible.
In this section, we formulate (cf. Section \ref{sec:thy-formulation}) and analyze (cf. Theorem~\ref{thm:test-loss-ltm}) a stylized linear regression problem that concretizes such a setting.

Recall that ResMem (Figure \ref{fig:cartoon}) involves first training a base neural network $f_\nn$, and then fitting the residual of $f_\nn$ on the same training data using a nearest-neighbor regressor $r_\knn$.
For feasibility of theoretical analysis, we simplify $f_\nn$ with a single layer linear neural network, i.e. linear regression, and we consider $1$-nearest neighbor instead of $k$-nearest neighbor to memorize the residual of this network.
Our results suggests that ResMem improves test-time generalization by augmenting the capacity of the base model with a non-parametric nearest-neighbor component.

\subsection{Assumptions and setting}
\label{sec:thy-formulation}

In this section, we present the setup and assumptions for the stylized linear regression problem.
We consider a setting where the function class that we minimize over does \emph{not} include the ground-truth function that relates the covariates to the response.
Therefore, even with infinite samples, the test loss will decay to a positive constant.
We exactly characterize the rate of decay, and show that it converges to $0$ under ResMem. Our analysis rests on the following assumptions.

\begin{assumption}[Distribution of covariates]
\label{ass:lin-dist}
The distribution of covariate $\bx\in\R^d$, denoted by $\P_\bx$, is the uniform distribution\footnote{For more general distributions, the theoretical result will depend on quantities like $\P_\bx(\cB(\bxt, h))$, where $\cB(\bxt, h)$ means a ball of radius $h$ that is centered at $\bxt$. We took uniform distribution for simplicity and to obtain exact dependence on $d$.} over a Euclidean norm ball centered at the origin of radius $\sqrt{d+2}$. The choice of radius ensures that $\E_{\bx\sim\P_\bx} \bx\bx^\sT = \bI$.
\end{assumption}

\begin{assumption}[Linear regression over norm ball]
\label{ass:lin-prob-def}
Consider the problem of learning a linear function $f_\star(\bx) = \<\bx, \btheta_\star\>$ with $\|\btheta_\star\|=1$ from training data $\{(\bx_i, y_i)\}_{i=1:n}$ where $\bx_i\overset{\rm i.i.d.}{\sim} \P_\bx$ and $y_i=f_\star(\bx_i)$ using the function class
\begin{equation}\label{eqn:linfunc}
  \cF = \{\bx\mapsto\<\bx,\btheta\>, \|\btheta\|<L\}.
\end{equation}
We assume $L<1$ so that the problem belongs to the ``hard generalization'' scenario discussed in Section \ref{sec:applicable}, where the hypothesis space is inadequate to fit the function on its own.
\end{assumption}

ResMem proceeds by first learning a linear function $f_n(\bx) = \<\btheta_n, \bx\>$ from $\cF$ through empirical risk minimization (ERM):
\begin{equation}
    \btheta_n = \underset{\|\btheta\|\leq L}{\operatorname{argmin}} \, \frac{1}{n} \sum_{i=1}^n \left[\<\bx_i, \btheta\>-y_i\right]^2.
\end{equation}
The empirical risk minimizer $f_n$ should be thought of as the analog of $f_\nn$ in the deep learning context.
It defines a ground-truth residual function $r_\star(\bx)=f_\star(\bx)-f_n(\bx)$.
Now we fix a test covariate $\bxt\sim\P_x$.
ResMem ``memorizes'' the residual function through the $1$-nearest neighbor to $\bxt$
\begin{equation}
    r_n(\bxt) = r_\star(\bxnn) = f_\star(\bxnn) - f_n(\bxnn),
\end{equation}
where $\bxnn$ is the nearest neighbor to $\bxt$ among the training covariates $\bx_1,\dots,\bx_n$:
\[
 \bxnn = \underset{\bx\in\{\bx_1, \dots, \bx_n\}}{\operatorname{argmin}} \, \|\bx - \bxt\|.
\]
The final prediction is
\begin{equation}
\label{eqn:lin-pred-rule}
f_{n}^{\alg}(\bxt) = f_n(\bxt) + r_n(\bxt).
\end{equation}
Observe that if $\bxt$ coincides with any training sample, 
$f_{n}^{\alg}(\bxt) = f_\star(\bxt)$, i.e., we have explicit memorization.
Note that we worked with $1$-nearest neighbor regressor for simplicity instead of the general $k$-nearest neighbor algorithm.
The effect of choosing different $k$ is not the main focus of this theoretical analysis.

\subsection{A decomposition of the target function}

Next, we introduce a decomposition of $f_\star$, which will help us analyze various components that make up the risk.
Define
\[
\begin{aligned}
\btheta_\infty
    &= \underset{\|\btheta\|\leq L}{\operatorname{argmin}} \,
       \E_{\bx\sim\P_x}\left[\<\btheta, \bx\> - \<\btheta_\star, \bx\>\right]^2,\\
    &= \underset{\|\btheta\|\leq L}{\operatorname{argmin}} \,
       \|\btheta-\btheta_\star\| = L \btheta_\star,
\end{aligned}
\]
which is what ERM learns in the limit of $n\rightarrow\infty$.
We can think of  $\btheta_\infty$ as the best function that ERM can learn.
Then, we can decompose $\btheta_\star$ into $\btheta_\star = \btheta_\infty + \btheta_\perp$, where $\btheta_\perp=\btheta_\star - \btheta_\infty$.
This decomposition can be generalized beyond linear regression.
Since $\btheta_\infty$ defines a function $f_\infty(\bx) = \<\bx, \btheta_\infty\>$, for general non-linear functions, the argument above can be generalized to the decomposition of $f_\star$ to an learnable and non-learnable part
\[f_\star =  f_\infty + f_\perp.\]
Intuitively, $f_\infty$ is the best function in $\mathcal{F}$ that ERM can learn, and $f_\perp$ is beyond the capacity of ERM due to the particular choice of function class.
ResMem approximates $f_\perp$ using the non-parametric nearest neighbor method, and therefore expanding the capacity of the original hypthesis class.

\subsection{A decomposition of the prediction error}
\label{sec:lin-risk-decomp}

We now introduce a decomposition of the prediction risk that reveals how $\alg$ algorithm boosts generalization.
Note that the prediction error of $\alg$ is
\begin{equation}
\label{eqn:lin-test-loss}
\E  \left[\left(f_{n}^{\alg}(\bxt)-f_\star(\bxt)\right)^2\right].
\end{equation}
It can be decomposed into two components: $\E  \left[f_{n}^{\alg}(\bxt)-f_\star(\bxt)\right]^2  \leq 3\times$
\[
[~ \underbrace{\E  (f_n(\bxt)-f_\infty(\bxt))^2
                    +\E  (f_n(\bxnn)-f_\infty(\bxnn))^2}_{T_1}
+\underbrace{\E  (f_\infty(\bxt)
                   -f_\star(\bxt)- f_\infty(\bxnn)
                   +f_\star(\bxnn))^2}_{T_2}
              ~].
\]
We provide the detail of the decomposition in Section \ref{sec:lin-decomp-detail}.
We can see that $T_1$ arises due to the difference between $f_n$ and $f_\infty$ (i.e., the estimation error), which, as we will show later, goes to $0$ as $n$ goes to infinity:
\[T_1\rightarrow 0 ~\text{as}~ n\rightarrow\infty.\]

On the other hand, $T_2$ arises due to the limited capacity of $\cF$.
It captures an irreducible error of the risk, which in general is \textbf{not} asymptotically zero.
However, because of the explicit memorization $\alg$ algorithm introduces ($\bxnn\rightarrow\bxt$ as $n\rightarrow\infty$), we also have
\[T_2\rightarrow 0 ~\text{as}~ n\rightarrow\infty.\]
This decomposition provides a statistical perspective on ResMem: it preserves the asymptotic consistency of $T_1$ as in classical learning problems while enforcing the asymptotic consistency of $T_2$ through the nearest-neighbor method.

\subsection{Main theoretical result}
Given the set up above, we are ready to state the main theoretical result of the paper, which characterizes the rate at which test risk of ResMem approaches 0. The proof is in Appendix \ref{sec:lin-proof}.

\begin{theorem}[Risk for $\alg$ algorithm]
\label{thm:test-loss-ltm}
For the problem defined in Assumption \ref{ass:lin-prob-def} with covariates distribution in Assumption \ref{ass:lin-dist}, the $\alg$ prediction rule $f_{n}^{\alg}(\bxt)$ defined in equation \eqref{eqn:lin-pred-rule} achieves risk \eqref{eqn:lin-test-loss}
\[
\begin{aligned}
\E \left[f_{n}^{\alg}(\bxt) - f_\star(\bxt)\right]^2
\lesssim d^2L^2n^{-2/3} + d^2(1-L)^2  \left[\frac{\log\left(n^{1/d}\right)}{n}\right]^{1/d},
\end{aligned}
\]
where $\lesssim$ denotes inequality up to a universal constant independent of $d, n$ and $L$.
\end{theorem}

The result includes contribution from two terms introduced in Section \ref{sec:lin-risk-decomp}:
\begin{itemize}
    \item $T_1\lesssim d^2L^2n^{-2/3}$ that arises due to the difference between $f_n$ and $f_\infty$.
    \item $T_2\lesssim \left[\log\left(n^{1/d}\right)/n\right]^{1/d}$ that vanishes as the nearest neighbor of the test point approaches the test point itself $\bxnn\rightarrow\bxt$.
\end{itemize}

The two terms $T_1$ and $T_2$ can be viewed as ``two stages of learning''.
Without the ResMem memorization component, we have the usual story of machine learning: $T_1\rightarrow 0$ at the usual parametric rate, and $T_2$ stays as an irreducible error, so the overall test error diminishes to a constant at a very fast rate.
With the introduction of nearest neighbor memorization procedure, $T_2$ can also be reduced to $0$ at a slower rate, whereas the fast decay of $T_1$ is still preserved.

This result shows why it is \textit{not favorable} to use the $k$-nearest neighbor component to memorize the response directly:
as a corollary of setting $L=0$ in Theorem \ref{thm:test-loss-ltm}, pure nearest neighbor would result in an overall slow rate of $\approx n^{-1/d}$.
However, with ResMem, we can enjoy benefit of having the test loss being asymptotically $0$, 
while also enjoying the fast rate of $n^{-2/3}$ for smaller sample sizes.

\section{Empirical results}
\label{sec:exp}
In this section, we present empirical results on image classification and language modeling that showcase the efficacy of ResMem.
In Section \ref{sec:background}, we present details of applying the ResMem algorithm to classification problems on real dataset.
In Section \ref{sec:exp-vision} and Section \ref{sec:exp-nlp}, we present the setup and the result for vision and language experiments, respectively.
In Section \ref{sec:emp-ana} we conduct an empirical analysis to explain where the improved accuracy of ResMem comes from.
Finally, in addition to evaluating the improvement ResMem algorithm over DeepNet itself, we compare ResMem with other reasonable baselines including \cite{knnlm} in Appendix \ref{sec:knnlm}.

\subsection{Details of ResMem algorithm for classification}
\label{sec:background}

We consider multi-class classification problems over instances $\XCal$ and labels $\YCal \defEq \{1, 2, \dots, L\} = [L]$.
Given training examples $S = \{ ( x_i, y_i ) \}_{i \in [n]} \in ( \XCal \times \YCal )^n$, 
the goal is to learn a scorer $f \colon \XCal \to \Real^L$ that,
given an instance,
assigns an affinity score for each label.
Such an $f$ should minimize the \emph{misclassification error} on test samples:
\begin{equation}
\label{eqn:L01}
L_{01}( f ) \defEq \Pr_{(x, y)}( y \neq {\tt pred}( f( x ) ) ), 
\end{equation}
where ${\tt pred}( z ) \defEq \argmax_{y' \in [L]} z_{y'}$,
and $\Pr$ is the distribution over labelled instances.
To achieve this, one typically minimizes the \emph{empirical loss}
\[
\hat{L}_\ell( f )\defEq \frac{1}{n} \sum_{i \in [n]} \ell( y_i, f( x_i ) ),
\]
where $\ell \colon [L] \times \Real^L \to \Real_+$ is a loss function.
Ideally, one would like to use 
$\ell_{01}( y, f( x ) ) \defEq 1( y \neq {\tt pred}( f( x ) ) )$;
for computational tractability, it is popular to instead use a \emph{surrogate loss}, such as the softmax cross-entropy.

Given the notation above, ResMem operates as follows:
\begin{enumerate}[leftmargin=5mm, itemsep=1mm, partopsep=0pt,parsep=0pt]
    \item \textbf{Train the base DeepNet.}
    Train a neural network $f_\nn$ on the training samples $S$ as usual.

    \item \textbf{Prepare the residual data.} Compute the \emph{residual}
    \footnote{For an overparameterized network that perfectly fits the training sample, the residuals will all be $0$. However, we are interested in either smaller networks or extremely large dataset where implicit memorization is infesible.}
    prediction of each training example as
    \[
    r_i = \textsf{onehot}(y_i) - \textsf{softmax}(f_\nn(x_i)/T),~\forall~i\in [n],
    \]
    where $\textsf{onehot} \colon \YCal \to \R^L$ is the standard encoding that maps the label to a probability vector.
    Here, $T$ is a hyperparameter corresponding to the ``temperature'' scaling of the softmax operation.
    Then, we employ the output of an intermediate layer of the base network $f_\nn$, denoted by $z_i = \phi(x_i)$, as an embedding for the training instance $x_i$.
    These embeddings are utilized for the nearest neighbor search in the next step.
    \item \textbf{Predict via memorized residuals.}
    To obtain a prediction on a test sample $\widetilde{x}\in\XCal$, first compute its embedding $\widetilde{z} = \phi(\widetilde{x})$. 
    Then, use soft $k$-nearest neighbor method to build a function $r_{\knn}$ defined by weights $\overline{w}_i(\widetilde{x})$:
    \begin{equation}
        r_{\knn}(\widetilde{x}) = \sum_{i=1}^n \overline{w}_i(\widetilde{x}) \cdot r_i.
    \end{equation}
    The weights $\overline{w}_i(\widetilde{x})$ satisfy $\sum_{i}\overline{w}_i(\widetilde{x}) = 1$, and
    are computed from raw weights $w_i$ as follows:
    \[
    w_i = \exp(-\|\widetilde{z}-z_i\|_2/\sigma), 
    \quad
    \overline{w}_i(\widetilde{x}) \propto \mathds{1}\left(w_i\geq w_{(k)}\right) w_i, 
    \]
    where $w_{(k)}$ represents the $k$-th largest entry of $w_i$'s. Note that $k$ and $\sigma$ are two hyperparameters that collectively controls the locality of nearest neighbor search.
\end{enumerate}

\noindent We make the final prediction based on the following scorer: 
\begin{equation}
\label{eqn:resmem-pred-classification}
f_\alg(\widetilde{x}) = \textsf{softmax}(f_\nn(\widetilde{x})/T) + r_{\knn}(\widetilde{x}).   
\end{equation}

\begin{remark}[Explicit memorization]
Smaller $k$ or $\sigma$ corresponds to putting higher weight on residuals of the closest neighboring training examples.
For sufficiently small $k$ and $\sigma$, $f_\alg$ achieves exact memorization of the training sample, i.e., ${\tt pred}(f_\alg(x_i))=y_i$ for all $i \in [n]$.
\end{remark}

\begin{remark}[Computation cost]
The vision experiments have moderate training sample size, so we perform exact nearest neighbor search and discuss the computation cost in Section \ref{sec:exp-vision}.
For language experiments, the training sample size is so large that the exact nearest neighbor computation is infesible, so we rely on \emph{approximate} nearest neighbor search discussed in Section \ref{sec:exp-vision}.
\end{remark}

\subsection{Image classification}
\label{sec:exp-vision}

\begin{figure*}[t]
    \centering
    \subfigure[\label{fig:cifar-arch}Test(left)/Training (right) v.s. architectures.]{
    \includegraphics[width=.24\textwidth]{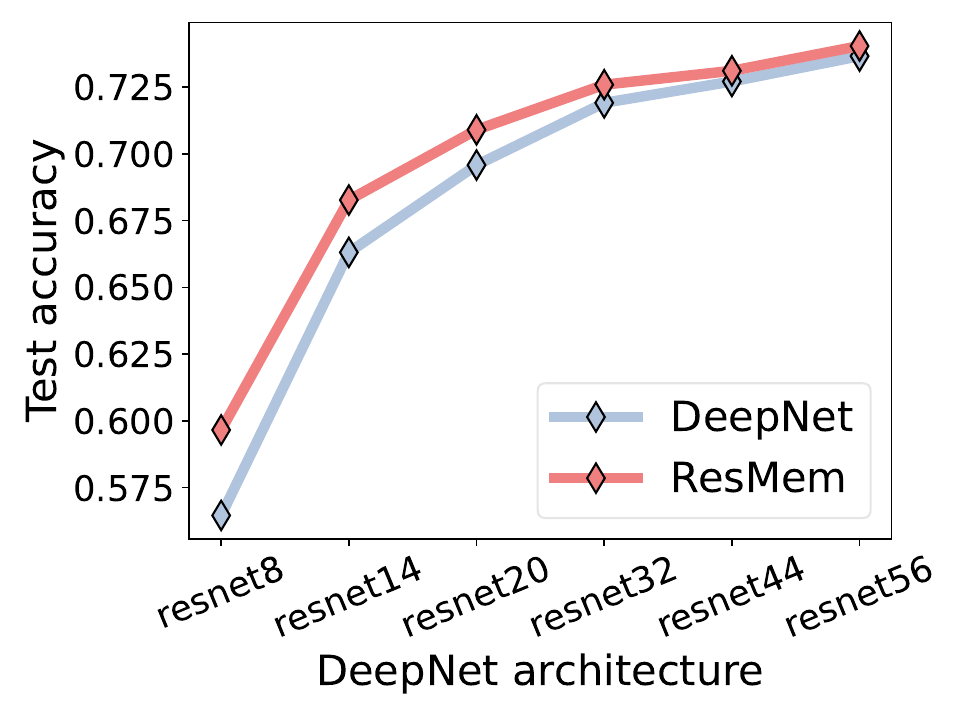}
    \includegraphics[width=.24\textwidth]{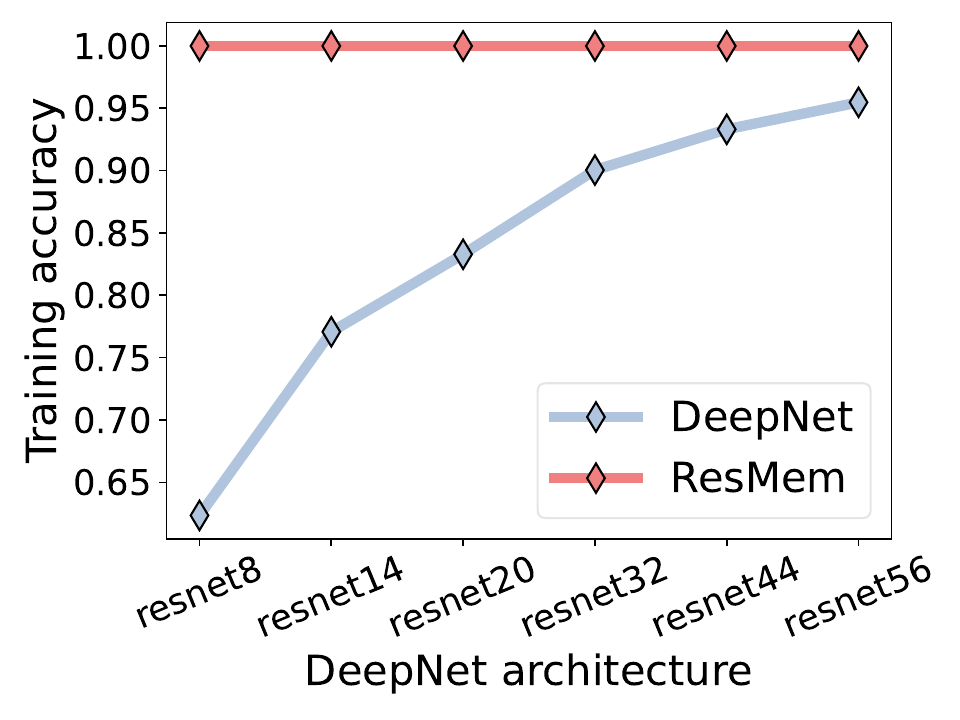}}
    \subfigure[\label{fig:cifar-sample}Test(left)/Training (right) acc. v.s. sample size.]{
    \includegraphics[width=.24\textwidth]{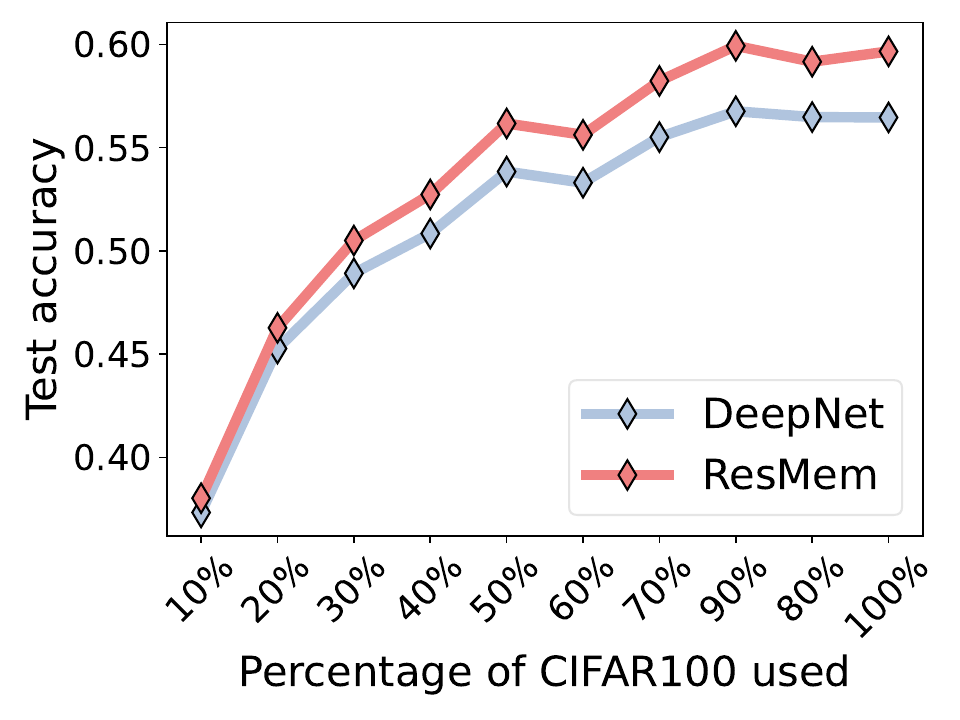}
    \includegraphics[width=.24\textwidth]{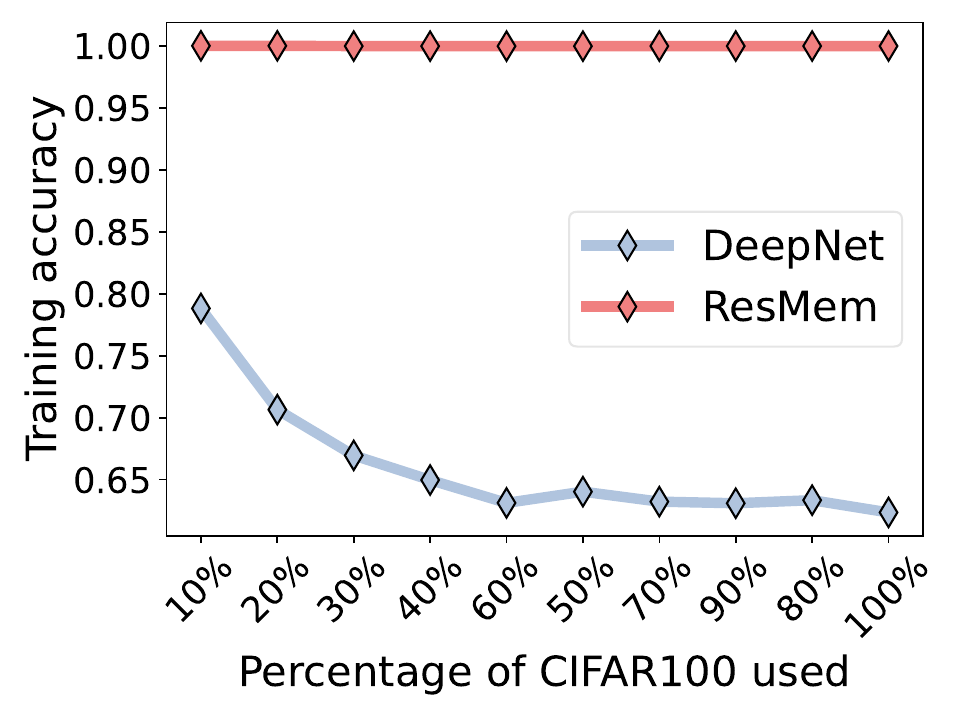}}
    \caption{
    ResMem improvement on CIFAR100 with respect to training sample size and deep network architecture.
    \textbf{(a):} Using progressively larger CIFAR-ResNet architecture.
    \textbf{(b):} Using $10\%, 20\%, \dots, 100\%$ of training data.
    }
    \label{fig:robustness-sample-arch}
\end{figure*}

In this subsection, we mainly consider image classification task on ResNet~\citep{resnet} with CIFAR100~\citep{cifar} dataset.
We provide additional ImageNet~\cite{ILSVRC15} results in Apendix \ref{sec:imagenet}.

\paragraph{Setup.}
We use CIFAR-ResNet-$\{8, 14, 20, 32, 44, 56\}$  as the base DeepNet.
For all six DeepNet training, we use SGD with batch size 128, trained for 256 epochs.
We use a peak learning rate 0.4, and momentum 0.9.
We warm up the learning rate linearly for the first 15 epochs, and decay the learning rate by $0.1$ after epochs $\{ 96, 192, 224 \}$.
For ResMem, we use the pre-logit layer as the image embedding, which has dimension 64.
For the nearest neighbor search (Step 3, Section \ref{sec:background}), we define the distance between two images to be the $\ell_2$ distance between their embeddings.
We use $\sigma=0.7$, $k=53$, and $T=1.4$ to compute the weights for the nearest neighbor regressor.
We provide the sensitivity analysis of test accuracy against ResMem parameters in Appendix \ref{sec:add-cifar} (cf. Figure \ref{fig:robustness-hyper}).

\paragraph{Results.}

The results for CIFAR-ResNet-$\{8, 14, 20, 32, 44, 56\}$ are reported in Figure \ref{fig:cifar-arch}.
We can see that ResMem boosts the test accuracy of CIFAR-ResNet8 from 56.46\% to \textbf{59.66\%}, which is between the base DeepNet test accuracy for CIFAR-ResNet8 and CIFAR-ResNet14.
To access the statistical reliability of the improvement, we repeat the CIFAR-ResNet-8 experiment 5 times over random initialization of DeepNet etc. We and that the average ResMem accuracy is $59\%$ with standard deviation $0.7\%$, and the average DeepNet accuracy is $56.5\%$ with standard deviation $0.8\%$.

Computationally, we estimate the CPU latency of a CIFAR-ResNet-8 to be 15.9 ms for a single test image.
By contrast, the $k$-NN step takes 4.8 ms for the same test image.
To contextualize the latency cost, the total cost of ResMem with ResNet-8 (15.9 ms + 4.8 ms) is lower than the cost of the next-sized model, i.e., ResNet-14 (26.2 ms).
Regarding the memory cost, for a batch size of 1 and images of size 32 x 32, a ResNet-8 (~68K params) requires 2.5MB, while a ResNet-14 (~128K params) requires 4MB. Embeddings from a ResNet-8 and ResNet-14 are both 64 dimensional. To embed the entire CIFAR100 training set (50K examples) requires ~15MB of disk space. 
\paragraph{Varying sample size.}
We repeat the above experiment on CIFAR-ResNet-8 with subsets ($10\%, 20\%, \dots, 100\%$) of CIFAR100 training data (subsampled uniformly across different classes).
The size of the index set for nearest-neighbor search is the same as the training set for base neural networks (e.g., model with 10\% CIFAR100 data also uses the same 10\%  data for nearest-neighbor search).
On the left (right) of Figure \ref{fig:cifar-sample}, we report the test (training) accuracy of ResMem and baseline DeepNet.
As a sanity check, we can see that ResMem always achieves perfect training accuracy, and the DeepNet training accuracy decreases as samples increase (since it's harder to fit larger dataset).
We can see that ResMem yields \emph{progressively larger margin of improvement when more data is used}.
This trend suggests a desirable property of ResMem: in real problems where the dataset is extremely large, ResMem is expected to bring even greater benefit.

\begin{figure*}[t]
    \centering
    \includegraphics[width=\textwidth]{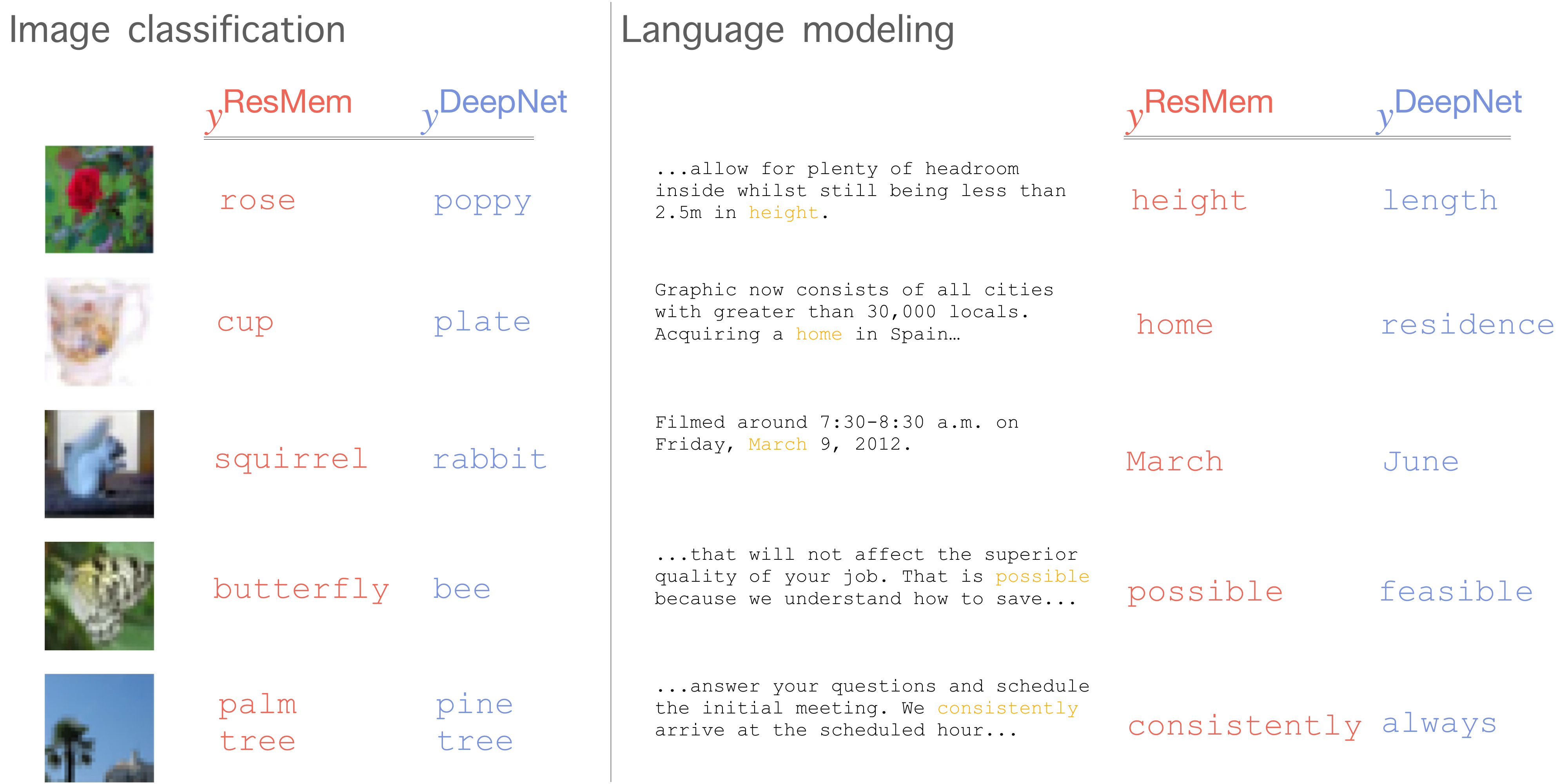}
    \caption{Examples from CIFAR100 and C4 test set with the property that 
    \textbf{(i)} $y^\textsf{ResMem}$ is correct;
    \textbf{(ii)} $y^\textsf{DeepNet}$ is wrong but close in meaning.
    We use red to denote the ResMem prediction and blue to denote DeepNet prediction.
    The DeepNet predictions capture \emph{coarse} structure (e.g., predicting \texttt{poppy} for a sample whose true label is \texttt{rose}),
    which can be refined by ResMem capturing the remaining \emph{fine-grained} structure.
    }
    \label{fig:cifar-nlp-analysis}
\end{figure*}

\subsection{Language modeling}
\label{sec:exp-nlp}

\paragraph{Setup.}
For the language experiment, we use a Decoder-Only T5-\{small, large\}~\citep{2020t5} model and C4~\citep{2020t5} dataset.
C4 is generated from scraping the internet and commonly used as a pretraining dataset or part of the pretraining mix.
We pre-trained the DeepNet on C4 training split with auto-regressive language modeling task.
For experimental efficiency, we used 1\% of the C4 training split (which corresponds to 1,639 million tokens) as the retrieval database, and extracted last transformer layer's pre-MLP, post-LayerNorm representations as the key embeddings for $k$NN search, and we created the query embeddings using the whole validation split and the same representation location. 
For each query, we retrieved 50 neighbors with $L_2$ distance using approximate nearest neighbor search algorithm ScaNN~\citep{avq_2020}.
We used the temperature $T = 1$ for the residual computation and $\sigma = 1$ for computing the neighbor weights. 
The predicted token is the one with highest probability, similar to greedy decoding, and we measured the prediction accuracy to match the vision experiments.

\paragraph{Results.}
On T5-small, ResMem boosts test accuracy from 38.01\% to {\bf 40.87\%}, which is around the accuracy (40.08\%) of a T5-base model without ResMem.
On T5-large, ResMem boosts the test accuracy from 44.8\% to {\bf 45.6\%}.
This demonstrates that with explicit memorization, we may leverage smaller base language models while reaping the performance benefits of large language models.
Computationally, as the index set is quite large (1.6 billion tokens), exact k-nearest neighbor search is infeasible. So we use the approximate nearest neighbor search algorithm ScaNN \cite{avq_2020} to reduce compute time.
Please see Appendix~\ref{sec:detail-nlp} for details on base model training and data processing.

\subsection{Where does the improvement come from?}
\label{sec:emp-ana}

In this section, we identify test samples that contributes to the accuracy improvement of CIFAR100 with CIFAR-ResNet-8 and C4 with T5-small.
Let ${\sf Gain}_{\sf ResMem}$ be the difference between the test accuracy of ResMem and baseline DeepNet:
\[
{\sf Gain}_{\sf ResMem} = L_{01}(f_\alg) - L_{01}(f_\nn),
\]
where $L_{01}$ is the misclassification error as defined in equation \eqref{eqn:L01}.
We offer a decomposition of ${\sf Gain}_{\sf ResMem}$ that sheds light into the mechanism behind ResMem.
For a test set $\{(x_i, y_i)\}_{i=1}^m$, let $y_i^\alg$ be the ResMem prediction on instance $x_i$ and let $y_i^\nn$ be the baseline neural network prediction on $x_i$.
When $y^\alg_i = y^\nn_i$, sample $x_i$ does not contribute to ${\sf Gain}_{\sf ResMem}$.
When $y^\alg_i\neq y^\nn_i$, this could arise 
either from the desirable event where the deep network misclassifies while ResMem classifies correctly;
or from the undesirable event where the ResMem misclassifies, while the deep network classifies correctly. 
These can be summarized by the \textsf{\small TPR} (\textit{true positive rate})
and \textsf{\small FPR} (\textit{false positive rate})
respectively:
\begin{equation}
    \textsf{\small TPR} = \frac{1}{m}\sum\nolimits_{i=1}^m \mathds{1}\{y^\nn_i\neq y_i ~\text{and}~ y^\alg_i = y_i\}.
\end{equation}
\begin{equation}
    \textsf{\small FPR} = \frac{1}{m}\sum\nolimits_{i=1}^m \mathds{1}\{y^\nn_i= y_i ~\text{and}~ y^\alg_i \neq y_i\}.
\end{equation}

Note that~${\sf Gain}_{\sf ResMem} = \textsf{\small TPR} - \textsf{\small FPR}.$
The decomposition of ${\sf Gain}_{\sf ResMem}$ says that the gain of ResMem came from the \textsf{\small TPR} samples, provided they outweigh the \textsf{\small FPR} samples.

On CIFAR-ResNet-8, we find \textsf{\small TPR}$=$5.89\% and \textsf{\small FPR}$=$2.70\%, leading to ${\sf Gain}_{\sf ResMem}$=3.19\%.
On T5-small with C4 validation split, we find \textsf{\small TPR}$=$5.37\% and \textsf{\small FPR}$=$2.44\%, leading to ${\sf Gain}_{\sf ResMem}$=2.93\%.

\paragraph{Analysis of \textsf{\small TPR} samples}
Focusing on the test samples where ResMem helps 
($y_i = y^\alg_i \neq y^\nn_i$), 
we identify a common underlying pattern: 
while the DeepNet makes an incorrect prediction, it still captures some coarse structure.
For example, in CIFAR100, 
one sample has correct label
$
y_i = y^\alg_i = \texttt{rose},
$
but the DeepNet predicts
$
y^\nn_i = \texttt{poppy}
$,
i.e., the label of a different type of flower.
(cf. Figure~\ref{fig:cifar-nlp-analysis}).
We find similar behavior for the language modeling task (cf. Figure~\ref{fig:cifar-nlp-analysis}).

This empirical analysis suggests the DeepNet in isolation can already learn some large scale structures, but is unable to make fine-grained distinctions.
This is where ResMem helps: 
\emph{ResMem helps memorize information in the training label that the DeepNet cannot learn.}

\paragraph{Additional insights from the decomposition.}
In this paper, we choose the ResMem hyperparameters that minimizes the test error on the validation set or, equivalently, maximize ${\sf Gain}_{\sf ResMem}$.
Inspired by the decomposition of ${\sf Gain}_{\sf ResMem}$, we propose an alternative hyperparameter selection procedure based on the following optimization problem:
\[
{\rm maximize}_{\textsf{\small FPR}(\texttt{hyperparam.})<0.05} \textsf{\small TPR}(\texttt{hyperparam.}),
\]
which ensures that ResMem modifies the DeepNet predictions in a more conservative manner.
In particular, bounding \textsf{\small FPR} implies that ResMem has minimal impact on the examples where DeepNet already makes correct predictions.
At the same time, a higher value of \textsf{\small TPR} corresponds maximizing the desirable occurrences where ResMem can correct a wrong prediction by DeepNet.

\section{Discussion and future works}
\label{sec:discussion}
\paragraph{Joint training of $k$NN and DeepNet.}
The current formulation of ResMem builds the base DeepNet and $k$NN components sequentially.
Consequently, the DeepNet is trained completely oblivious to the fact that there is a subsequent $k$NN model that will memorize its residuals.
A natural direction of future work is to consider the \emph{joint} training of DeepNet and $k$NN, so that the models can dynamically interact during training to determine which portion of label is for DeepNet to learn, and the remaining is for $k$NN to memorize.

To explore the role of training during the first stage, we re-evaluate the CIFAR-ResNet-8 experiment by stopping DeepNet training at different epochs (Table \ref{table:joint}).
\begin{table}[h]
\centering
\caption{Comparison of DeepNet and ResMem accuracy over epochs on CIFAR-ResNet-8 experiment.}
\label{table:joint}
\begin{tabular}{lccccc}
\toprule
\#epoch & 128 & 160 & 192 & 224 & 256 \\
\midrule
DeepNet acc. & 34.0\% & 56.2\% & 55.6\% & 57.2\% & 56.6\% \\
ResMem acc. & 49.3\% & 60.2\% & 58.6\% & 59.2\% & 59.5\% \\
\bottomrule
\end{tabular}
\end{table}
We can see that when the \#epoch is small, ResMem has a dramatic improvement in accuracy. One of the key roles of the first training phase is to learn good representations of the training data so the nearest neighbor retrieval is performed on more meaningful representations.
This simple experiments suggests that the proposed direction has the potential to dramatically reduce the training time of DeepNet -- while obtaining similar test accuracy with the help of ResMem.

\paragraph{Calibration of ResMem.}
A potential problem with applying ResMem to classification is \emph{scorer mis-calibration}.
The output of the ResMem prediction vector $f_\alg(x)$~\eqref{eqn:resmem-pred-classification} is not guaranteed to lie on the probability simplex.
This is not an issue when we only care about the predicted class membership, since we take the argmax of $f_\alg(x)$.
However, this limitation hinders us to access the \emph{confidence} of the ResMem prediction.
To remedy this, a possible future work is to consider alternative notions of residual.
For example, we can do memorization in the logit space instead of the probability space.
Then, the one-hot encoding of the true label may be replaced by class mean when defining the residual.

\paragraph{Distribution shift.}
Finally, ResMem can be a promising approach to tackle test-time covariate shift.
The nearest neighbor modifies the prediction of DeepNet based on the training covariate that are closer to the test covariate, making the algorithm more \emph{adaptive} to the specific test covariate~\citep{sun19ttt}.

\section*{Acknowledgements}
Part of the work is done while Zitong Yang is at Google Research, New York.
We would like to thank Chong You, Yu Sun, Yaodong Yu and anonymous reviewers for their feedback on the final draft.
Zitong Yang would like to thank Shuangping Li for discussion regarding the proof of Lemma \ref{lem:nearestneighbor}.
Zitong Yang would also like to acknowledge the support of Albion Walter Hewlett Stanford Graduate Fellowship.

\bibliographystyle{plainnat}
\bibliography{reference}

\begin{thebibliography}{55}
\providecommand{\natexlab}[1]{#1}
\providecommand{\url}[1]{\texttt{#1}}
\expandafter\ifx\csname urlstyle\endcsname\relax
  \providecommand{\doi}[1]{doi: #1}\else
  \providecommand{\doi}{doi: \begingroup \urlstyle{rm}\Url}\fi

\bibitem[Alley(1987)]{Alley:1987}
William~M. Alley.
\newblock A note on stagewise regression.
\newblock \emph{The American Statistician}, 41\penalty0 (2):\penalty0 132--134,
  1987.
\newblock \doi{10.1080/00031305.1987.10475461}.
\newblock URL
  \url{https://www.tandfonline.com/doi/abs/10.1080/00031305.1987.10475461}.

\bibitem[Arpit et~al.(2017)Arpit, Jastrzebski, Ballas, Krueger, Bengio, Kanwal,
  Maharaj, Fischer, Courville, Bengio, and Lacoste-Julien]{Arpit:2017}
Devansh Arpit, Stanis\l{}aw Jastrzebski, Nicolas Ballas, David Krueger,
  Emmanuel Bengio, Maxinder~S. Kanwal, Tegan Maharaj, Asja Fischer, Aaron
  Courville, Yoshua Bengio, and Simon Lacoste-Julien.
\newblock A closer look at memorization in deep networks.
\newblock In \emph{Proceedings of the 34th International Conference on Machine
  Learning - Volume 70}, ICML'17, pages 233--242. JMLR.org, 2017.

\bibitem[Bartlett et~al.(2017)Bartlett, Foster, and Telgarsky]{Bartlett:2017}
Peter~L. Bartlett, Dylan~J. Foster, and Matus Telgarsky.
\newblock Spectrally-normalized margin bounds for neural networks.
\newblock In Isabelle Guyon, Ulrike von Luxburg, Samy Bengio, Hanna~M. Wallach,
  Rob Fergus, S.~V.~N. Vishwanathan, and Roman Garnett, editors, \emph{Advances
  in Neural Information Processing Systems 30: Annual Conference on Neural
  Information Processing Systems 2017, December 4-9, 2017, Long Beach, CA,
  {USA}}, pages 6240--6249, 2017.

\bibitem[Bartlett et~al.(2020)Bartlett, Long, Lugosi, and
  Tsigler]{Bartlett:2020}
Peter~L. Bartlett, Philip~M. Long, G{\'a}bor Lugosi, and Alexander Tsigler.
\newblock Benign overfitting in linear regression.
\newblock \emph{Proceedings of the National Academy of Sciences}, 117\penalty0
  (48):\penalty0 30063--30070, 2020.
\newblock \doi{10.1073/pnas.1907378117}.
\newblock URL \url{https://www.pnas.org/doi/abs/10.1073/pnas.1907378117}.

\bibitem[Belkin et~al.(2018)Belkin, Hsu, and Mitra]{Belkin:2018}
Mikhail Belkin, Daniel Hsu, and Partha~P. Mitra.
\newblock Overfitting or perfect fitting? risk bounds for classification and
  regression rules that interpolate.
\newblock In \emph{Proceedings of the 32nd International Conference on Neural
  Information Processing Systems}, NIPS'18, pages 2306--2317, Red Hook, NY,
  USA, 2018. Curran Associates Inc.

\bibitem[Borgeaud et~al.(2021)Borgeaud, Mensch, Hoffmann, Cai, Rutherford,
  Millican, van~den Driessche, Lespiau, Damoc, Clark, de~Las~Casas, Guy,
  Menick, Ring, Hennigan, Huang, Maggiore, Jones, Cassirer, Brock, Paganini,
  Irving, Vinyals, Osindero, Simonyan, Rae, Elsen, and Sifre]{Borgeaud:2021}
Sebastian Borgeaud, Arthur Mensch, Jordan Hoffmann, Trevor Cai, Eliza
  Rutherford, Katie Millican, George van~den Driessche, Jean{-}Baptiste
  Lespiau, Bogdan Damoc, Aidan Clark, Diego de~Las~Casas, Aurelia Guy, Jacob
  Menick, Roman Ring, Tom Hennigan, Saffron Huang, Loren Maggiore, Chris Jones,
  Albin Cassirer, Andy Brock, Michela Paganini, Geoffrey Irving, Oriol Vinyals,
  Simon Osindero, Karen Simonyan, Jack~W. Rae, Erich Elsen, and Laurent Sifre.
\newblock Improving language models by retrieving from trillions of tokens.
\newblock \emph{CoRR}, abs/2112.04426, 2021.
\newblock URL \url{https://arxiv.org/abs/2112.04426}.

\bibitem[Bottou and Vapnik(1992)]{locallearning}
Léon Bottou and Vladimir Vapnik.
\newblock {Local Learning Algorithms}.
\newblock \emph{Neural Computation}, 4\penalty0 (6):\penalty0 888--900, 11
  1992.
\newblock ISSN 0899-7667.
\newblock \doi{10.1162/neco.1992.4.6.888}.
\newblock URL \url{https://doi.org/10.1162/neco.1992.4.6.888}.

\bibitem[Brutzkus et~al.(2018)Brutzkus, Globerson, Malach, and
  Shalev-Shwartz]{Brutzkus:2018}
Alon Brutzkus, Amir Globerson, Eran Malach, and Shai Shalev-Shwartz.
\newblock {SGD} learns over-parameterized networks that provably generalize on
  linearly separable data.
\newblock In \emph{International Conference on Learning Representations}, 2018.
\newblock URL \url{https://openreview.net/forum?id=rJ33wwxRb}.

\bibitem[Bucil\v{a} et~al.(2006)Bucil\v{a}, Caruana, and
  Niculescu-Mizil]{Bucilua:2006}
Cristian Bucil\v{a}, Rich Caruana, and Alexandru Niculescu-Mizil.
\newblock Model compression.
\newblock In \emph{Proceedings of the 12th ACM SIGKDD International Conference
  on Knowledge Discovery and Data Mining}, KDD '06, pages 535--541, New York,
  NY, USA, 2006. ACM.

\bibitem[Chaudhuri and Dasgupta(2014)]{nnrates}
Kamalika Chaudhuri and Sanjoy Dasgupta.
\newblock Rates of convergence for nearest neighbor classification.
\newblock In Z.~Ghahramani, M.~Welling, C.~Cortes, N.~Lawrence, and K.Q.
  Weinberger, editors, \emph{Advances in Neural Information Processing
  Systems}, volume~27. Curran Associates, Inc., 2014.
\newblock URL
  \url{https://proceedings.neurips.cc/paper/2014/file/db957c626a8cd7a27231adfbf51e20eb-Paper.pdf}.

\bibitem[Cheng et~al.(2022)Cheng, Duchi, and Kuditipudi]{CJK2022Mem}
Chen Cheng, John Duchi, and Rohith Kuditipudi.
\newblock Memorize to generalize: on the necessity of interpolation in high
  dimensional linear regression, 2022.
\newblock URL \url{https://arxiv.org/abs/2202.09889}.

\bibitem[Chowdhery et~al.(2022)Chowdhery, Narang, Devlin, Bosma, Mishra,
  Roberts, Barham, Chung, Sutton, Gehrmann, Schuh, Shi, Tsvyashchenko, Maynez,
  Rao, Barnes, Tay, Shazeer, Prabhakaran, Reif, Du, Hutchinson, Pope, Bradbury,
  Austin, Isard, Gur-Ari, Yin, Duke, Levskaya, Ghemawat, Dev, Michalewski,
  Garcia, Misra, Robinson, Fedus, Zhou, Ippolito, Luan, Lim, Zoph, Spiridonov,
  Sepassi, Dohan, Agrawal, Omernick, Dai, Pillai, Pellat, Lewkowycz, Moreira,
  Child, Polozov, Lee, Zhou, Wang, Saeta, Diaz, Firat, Catasta, Wei,
  Meier-Hellstern, Eck, Dean, Petrov, and Fiedel]{Chowdhery:2022}
Aakanksha Chowdhery, Sharan Narang, Jacob Devlin, Maarten Bosma, Gaurav Mishra,
  Adam Roberts, Paul Barham, Hyung~Won Chung, Charles Sutton, Sebastian
  Gehrmann, Parker Schuh, Kensen Shi, Sasha Tsvyashchenko, Joshua Maynez,
  Abhishek Rao, Parker Barnes, Yi~Tay, Noam Shazeer, Vinodkumar Prabhakaran,
  Emily Reif, Nan Du, Ben Hutchinson, Reiner Pope, James Bradbury, Jacob
  Austin, Michael Isard, Guy Gur-Ari, Pengcheng Yin, Toju Duke, Anselm
  Levskaya, Sanjay Ghemawat, Sunipa Dev, Henryk Michalewski, Xavier Garcia,
  Vedant Misra, Kevin Robinson, Liam Fedus, Denny Zhou, Daphne Ippolito, David
  Luan, Hyeontaek Lim, Barret Zoph, Alexander Spiridonov, Ryan Sepassi, David
  Dohan, Shivani Agrawal, Mark Omernick, Andrew~M. Dai,
  Thanumalayan~Sankaranarayana Pillai, Marie Pellat, Aitor Lewkowycz, Erica
  Moreira, Rewon Child, Oleksandr Polozov, Katherine Lee, Zongwei Zhou, Xuezhi
  Wang, Brennan Saeta, Mark Diaz, Orhan Firat, Michele Catasta, Jason Wei,
  Kathy Meier-Hellstern, Douglas Eck, Jeff Dean, Slav Petrov, and Noah Fiedel.
\newblock Palm: Scaling language modeling with pathways, 2022.
\newblock URL \url{https://arxiv.org/abs/2204.02311}.

\bibitem[Cohen et~al.(2018)Cohen, Sapiro, and Giryes]{arXivDnnKnn}
Gilad Cohen, Guillermo Sapiro, and Raja Giryes.
\newblock Dnn or k-nn: That is the generalize vs. memorize question, 2018.
\newblock URL \url{https://arxiv.org/abs/1805.06822}.

\bibitem[Cover and Hart(1967)]{NNCover}
T.~Cover and P.~Hart.
\newblock Nearest neighbor pattern classification.
\newblock \emph{IEEE Transactions on Information Theory}, 13\penalty0
  (1):\penalty0 21--27, 1967.
\newblock \doi{10.1109/TIT.1967.1053964}.

\bibitem[Dziugaite and Roy(2017)]{Dziugaite:2017}
Gintare~Karolina Dziugaite and Daniel~M. Roy.
\newblock Computing nonvacuous generalization bounds for deep (stochastic)
  neural networks with many more parameters than training data.
\newblock In \emph{Proceedings of the 33rd Annual Conference on Uncertainty in
  Artificial Intelligence (UAI)}, 2017.

\bibitem[Feldman(2019)]{Feldman:2019}
Vitaly Feldman.
\newblock Does learning require memorization? {A} short tale about a long tail.
\newblock \emph{CoRR}, abs/1906.05271, 2019.
\newblock URL \url{http://arxiv.org/abs/1906.05271}.

\bibitem[Feldman(2020)]{Feldman2020}
Vitaly Feldman.
\newblock \emph{Does Learning Require Memorization? A Short Tale about a Long
  Tail}, page 954–959.
\newblock Association for Computing Machinery, New York, NY, USA, 2020.
\newblock ISBN 9781450369794.
\newblock URL \url{https://doi.org/10.1145/3357713.3384290}.

\bibitem[Feldman and Zhang(2020{\natexlab{a}})]{Feldman2020_DiscoveringLT}
Vitaly Feldman and Chiyuan Zhang.
\newblock What neural networks memorize and why: Discovering the long tail via
  influence estimation.
\newblock In H.~Larochelle, M.~Ranzato, R.~Hadsell, M.F. Balcan, and H.~Lin,
  editors, \emph{Advances in Neural Information Processing Systems}, volume~33,
  pages 2881--2891. Curran Associates, Inc., 2020{\natexlab{a}}.
\newblock URL
  \url{https://proceedings.neurips.cc/paper/2020/file/1e14bfe2714193e7af5abc64ecbd6b46-Paper.pdf}.

\bibitem[Feldman and Zhang(2020{\natexlab{b}})]{FeldmanZhang2020}
Vitaly Feldman and Chiyuan Zhang.
\newblock What neural networks memorize and why: Discovering the long tail via
  influence estimation.
\newblock In \emph{Proceedings of the 34th International Conference on Neural
  Information Processing Systems}, NIPS'20, Red Hook, NY, USA,
  2020{\natexlab{b}}. Curran Associates Inc.
\newblock ISBN 9781713829546.

\bibitem[Freund and Schapire(1995)]{Freund:1997}
Yoav Freund and Robert~E. Schapire.
\newblock A desicion-theoretic generalization of on-line learning and an
  application to boosting.
\newblock In Paul Vit{\'a}nyi, editor, \emph{Computational Learning Theory},
  pages 23--37, Berlin, Heidelberg, 1995. Springer Berlin Heidelberg.
\newblock ISBN 978-3-540-49195-8.

\bibitem[Friedman et~al.(2000)Friedman, Hastie, and Tibshirani]{Friedman:2000}
Jerome Friedman, Trevor Hastie, and Robert Tibshirani.
\newblock {Additive logistic regression: a statistical view of boosting (With
  discussion and a rejoinder by the authors)}.
\newblock \emph{The Annals of Statistics}, 28\penalty0 (2):\penalty0 337 --
  407, 2000.
\newblock \doi{10.1214/aos/1016218223}.
\newblock URL \url{https://doi.org/10.1214/aos/1016218223}.

\bibitem[Gionis et~al.(1999)Gionis, Indyk, and Motwani]{Gionis1999SimilaritySI}
A.~Gionis, Piotr Indyk, and Rajeev Motwani.
\newblock Similarity search in high dimensions via hashing.
\newblock In \emph{Very Large Data Bases Conference}, 1999.

\bibitem[Goldberger and Jochems(1961)]{Goldberger:1961}
Arthur~S. Goldberger and D.~B. Jochems.
\newblock Note on stepwise least squares.
\newblock \emph{Journal of the American Statistical Association}, 56\penalty0
  (293):\penalty0 105--110, 1961.
\newblock \doi{10.1080/01621459.1961.10482095}.
\newblock URL
  \url{https://www.tandfonline.com/doi/abs/10.1080/01621459.1961.10482095}.

\bibitem[Guo et~al.(2020)Guo, Sun, Lindgren, Geng, Simcha, Chern, and
  Kumar]{avq_2020}
Ruiqi Guo, Philip Sun, Erik Lindgren, Quan Geng, David Simcha, Felix Chern, and
  Sanjiv Kumar.
\newblock Accelerating large-scale inference with anisotropic vector
  quantization.
\newblock In \emph{International Conference on Machine Learning}, 2020.
\newblock URL \url{https://arxiv.org/abs/1908.10396}.

\bibitem[Guu et~al.(2020)Guu, Lee, Tung, Pasupat, and Chang]{Guu:2020}
Kelvin Guu, Kenton Lee, Zora Tung, Panupong Pasupat, and Ming-Wei Chang.
\newblock Realm: Retrieval-augmented language model pre-training.
\newblock In \emph{Proceedings of the 37th International Conference on Machine
  Learning}, ICML'20. JMLR.org, 2020.

\bibitem[Hastie et~al.(2001)Hastie, Tibshirani, and Friedman]{Hastie:2001}
Trevor Hastie, Robert Tibshirani, and Jerome Friedman.
\newblock \emph{The Elements of Statistical Learning}.
\newblock Springer Series in Statistics. Springer New York Inc., New York, NY,
  USA, 2001.

\bibitem[He et~al.(2016)He, Zhang, Ren, and Sun]{resnet}
Kaiming He, Xiangyu Zhang, Shaoqing Ren, and Jian Sun.
\newblock Deep residual learning for image recognition.
\newblock In \emph{2016 IEEE Conference on Computer Vision and Pattern
  Recognition (CVPR)}, pages 770--778, 2016.
\newblock \doi{10.1109/CVPR.2016.90}.

\bibitem[Hinton et~al.(2015{\natexlab{a}})Hinton, Vinyals, and
  Dean]{distillation}
Geoffrey Hinton, Oriol Vinyals, and Jeff Dean.
\newblock Distilling the knowledge in a neural network, 2015{\natexlab{a}}.
\newblock URL \url{https://arxiv.org/abs/1503.02531}.

\bibitem[Hinton et~al.(2015{\natexlab{b}})Hinton, Vinyals, and
  Dean]{Hinton:2015}
Geoffrey~E. Hinton, Oriol Vinyals, and Jeffrey Dean.
\newblock Distilling the knowledge in a neural network.
\newblock \emph{CoRR}, abs/1503.02531, 2015{\natexlab{b}}.

\bibitem[Khandelwal et~al.(2020{\natexlab{a}})Khandelwal, Levy, Jurafsky,
  Zettlemoyer, and Lewis]{Khandelwal:2020}
Urvashi Khandelwal, Omer Levy, Dan Jurafsky, Luke Zettlemoyer, and Mike Lewis.
\newblock Generalization through memorization: Nearest neighbor language
  models.
\newblock In \emph{International Conference on Learning Representations},
  2020{\natexlab{a}}.
\newblock URL \url{https://openreview.net/forum?id=HklBjCEKvH}.

\bibitem[Khandelwal et~al.(2020{\natexlab{b}})Khandelwal, Levy, Jurafsky,
  Zettlemoyer, and Lewis]{knnlm}
Urvashi Khandelwal, Omer Levy, Dan Jurafsky, Luke Zettlemoyer, and Mike Lewis.
\newblock Generalization through memorization: Nearest neighbor language
  models.
\newblock In \emph{ICLR}, 2020{\natexlab{b}}.
\newblock URL \url{https://openreview.net/forum?id=HklBjCEKvH}.

\bibitem[Knuth(1973)]{KnuthAOP}
Donald Knuth.
\newblock \emph{The Art Of Computer Programming, vol. 3: Sorting And
  Searching}.
\newblock Addison-Wesley, 1973.

\bibitem[Krizhevsky(2009)]{cifar}
Alex Krizhevsky.
\newblock Learning multiple layers of features from tiny images.
\newblock Technical report, CIFAR, 2009.

\bibitem[Lample et~al.(2019)Lample, Sablayrolles, Ranzato, Denoyer, and
  J{\'{e}}gou]{Lample:2019}
Guillaume Lample, Alexandre Sablayrolles, Marc'Aurelio Ranzato, Ludovic
  Denoyer, and Herv{\'{e}} J{\'{e}}gou.
\newblock Large memory layers with product keys.
\newblock In Hanna~M. Wallach, Hugo Larochelle, Alina Beygelzimer, Florence
  d'Alch{\'{e}}{-}Buc, Emily~B. Fox, and Roman Garnett, editors, \emph{Advances
  in Neural Information Processing Systems 32: Annual Conference on Neural
  Information Processing Systems 2019, NeurIPS 2019, December 8-14, 2019,
  Vancouver, BC, Canada}, pages 8546--8557, 2019.
\newblock URL
  \url{https://proceedings.neurips.cc/paper/2019/hash/9d8df73a3cfbf3c5b47bc9b50f214aff-Abstract.html}.

\bibitem[Li et~al.(2022)Li, Guo, and Kumar]{li2022decoupled}
Zonglin Li, Ruiqi Guo, and Sanjiv Kumar.
\newblock Decoupled context processing for context augmented language modeling.
\newblock In Alice~H. Oh, Alekh Agarwal, Danielle Belgrave, and Kyunghyun Cho,
  editors, \emph{Advances in Neural Information Processing Systems}, 2022.
\newblock URL \url{https://openreview.net/forum?id=02dbnEbEFn}.

\bibitem[Liang and Rakhlin(2020)]{Liang:2020}
Tengyuan Liang and Alexander Rakhlin.
\newblock {Just interpolate: Kernel ``Ridgeless'' regression can generalize}.
\newblock \emph{The Annals of Statistics}, 48\penalty0 (3):\penalty0 1329 --
  1347, 2020.
\newblock \doi{10.1214/19-AOS1849}.
\newblock URL \url{https://doi.org/10.1214/19-AOS1849}.

\bibitem[Menghani(2021)]{Menghani:2021}
Gaurav Menghani.
\newblock Efficient deep learning: {A} survey on making deep learning models
  smaller, faster, and better.
\newblock \emph{CoRR}, abs/2106.08962, 2021.
\newblock URL \url{https://arxiv.org/abs/2106.08962}.

\bibitem[Montanari and Zhong(2020)]{Montanari:2020}
Andrea Montanari and Yiqiao Zhong.
\newblock The interpolation phase transition in neural networks: Memorization
  and generalization under lazy training, 2020.
\newblock URL \url{https://arxiv.org/abs/2007.12826}.

\bibitem[Muja and Lowe(2009)]{Muja2009FastAN}
Marius Muja and David~G. Lowe.
\newblock Fast approximate nearest neighbors with automatic algorithm
  configuration.
\newblock In \emph{International Conference on Computer Vision Theory and
  Applications}, 2009.

\bibitem[Neyshabur et~al.(2019)Neyshabur, Li, Bhojanapalli, LeCun, and
  Srebro]{Neyshabur:2019}
Behnam Neyshabur, Zhiyuan Li, Srinadh Bhojanapalli, Yann LeCun, and Nathan
  Srebro.
\newblock The role of over-parametrization in generalization of neural
  networks.
\newblock In \emph{7th International Conference on Learning Representations,
  {ICLR} 2019, New Orleans, LA, USA, May 6-9, 2019}. OpenReview.net, 2019.

\bibitem[Panigrahy et~al.(2021)Panigrahy, Wang, and Zaheer]{Panigrahy:2021}
Rina Panigrahy, Xin Wang, and Manzil Zaheer.
\newblock Sketch based memory for neural networks.
\newblock In Arindam Banerjee and Kenji Fukumizu, editors, \emph{Proceedings of
  The 24th International Conference on Artificial Intelligence and Statistics},
  volume 130 of \emph{Proceedings of Machine Learning Research}, pages
  3169--3177. PMLR, 13--15 Apr 2021.
\newblock URL \url{https://proceedings.mlr.press/v130/panigrahy21a.html}.

\bibitem[Raffel et~al.(2020)Raffel, Shazeer, Roberts, Lee, Narang, Matena,
  Zhou, Li, and Liu]{2020t5}
Colin Raffel, Noam Shazeer, Adam Roberts, Katherine Lee, Sharan Narang, Michael
  Matena, Yanqi Zhou, Wei Li, and Peter~J. Liu.
\newblock Exploring the limits of transfer learning with a unified text-to-text
  transformer.
\newblock \emph{Journal of Machine Learning Research}, 21\penalty0
  (140):\penalty0 1--67, 2020.
\newblock URL \url{http://jmlr.org/papers/v21/20-074.html}.

\bibitem[Russakovsky et~al.(2015)Russakovsky, Deng, Su, Krause, Satheesh, Ma,
  Huang, Karpathy, Khosla, Bernstein, Berg, and Fei-Fei]{ILSVRC15}
Olga Russakovsky, Jia Deng, Hao Su, Jonathan Krause, Sanjeev Satheesh, Sean Ma,
  Zhiheng Huang, Andrej Karpathy, Aditya Khosla, Michael Bernstein,
  Alexander~C. Berg, and Li~Fei-Fei.
\newblock {ImageNet Large Scale Visual Recognition Challenge}.
\newblock \emph{International Journal of Computer Vision (IJCV)}, 115\penalty0
  (3):\penalty0 211--252, 2015.
\newblock \doi{10.1007/s11263-015-0816-y}.

\bibitem[Salakhutdinov and Hinton(2007)]{pmlr-v2-salakhutdinov07a}
Ruslan Salakhutdinov and Geoff Hinton.
\newblock Learning a nonlinear embedding by preserving class neighbourhood
  structure.
\newblock In Marina Meila and Xiaotong Shen, editors, \emph{Proceedings of the
  Eleventh International Conference on Artificial Intelligence and Statistics},
  volume~2 of \emph{Proceedings of Machine Learning Research}, pages 412--419,
  San Juan, Puerto Rico, 21--24 Mar 2007. PMLR.
\newblock URL \url{https://proceedings.mlr.press/v2/salakhutdinov07a.html}.

\bibitem[Sandler et~al.(2018)Sandler, Howard, Zhu, Zhmoginov, and
  Chen]{Sandler:2018}
Mark Sandler, Andrew Howard, Menglong Zhu, Andrey Zhmoginov, and Liang-Chieh
  Chen.
\newblock Mobilenetv2: Inverted residuals and linear bottlenecks.
\newblock In \emph{2018 IEEE/CVF Conference on Computer Vision and Pattern
  Recognition}, pages 4510--4520, 2018.
\newblock \doi{10.1109/CVPR.2018.00474}.

\bibitem[Sun et~al.(2020)Sun, Wang, Zhuang, Miller, Hardt, and Efros]{sun19ttt}
Yu~Sun, Xiaolong Wang, Liu Zhuang, John Miller, Moritz Hardt, and Alexei~A.
  Efros.
\newblock Test-time training with self-supervision for generalization under
  distribution shifts.
\newblock In \emph{ICML}, 2020.

\bibitem[Tibshirani(2015)]{Tibshirani:2015}
Ryan~J. Tibshirani.
\newblock A general framework for fast stagewise algorithms.
\newblock \emph{J. Mach. Learn. Res.}, 16\penalty0 (1):\penalty0 2543–2588,
  jan 2015.
\newblock ISSN 1532-4435.

\bibitem[Vapnik and Izmailov(2021)]{Vapnik:2021}
Vladimir Vapnik and Rauf Izmailov.
\newblock Reinforced {SVM} method and memorization mechanisms.
\newblock \emph{Pattern Recognition}, 119:\penalty0 108018, 2021.
\newblock ISSN 0031-3203.
\newblock \doi{https://doi.org/10.1016/j.patcog.2021.108018}.
\newblock URL
  \url{https://www.sciencedirect.com/science/article/pii/S0031320321002053}.

\bibitem[Wainwright(2019)]{wainwright2019high}
Martin~J Wainwright.
\newblock \emph{High-dimensional statistics: A non-asymptotic viewpoint},
  volume~48.
\newblock Cambridge University Press, 2019.

\bibitem[Wang et~al.(2021)Wang, Muthukumar, and Thrampoulidis]{Wang:2021b}
Ke~Wang, Vidya Muthukumar, and Christos Thrampoulidis.
\newblock Benign overfitting in multiclass classification: All roads lead to
  interpolation, 2021.
\newblock URL \url{https://arxiv.org/abs/2106.10865}.

\bibitem[Wang and Shao(2022)]{Wang:2022}
Zhen Wang and Yuan-Hai Shao.
\newblock Generalization-memorization machines, 2022.
\newblock URL \url{https://arxiv.org/abs/2207.03976}.

\bibitem[Yang et~al.(2020)Yang, Yu, You, Steinhardt, and Ma]{pmlr-v119-yang20j}
Zitong Yang, Yaodong Yu, Chong You, Jacob Steinhardt, and Yi~Ma.
\newblock Rethinking bias-variance trade-off for generalization of neural
  networks.
\newblock In Hal~Daumé III and Aarti Singh, editors, \emph{Proceedings of the
  37th International Conference on Machine Learning}, volume 119 of
  \emph{Proceedings of Machine Learning Research}, pages 10767--10777. PMLR,
  13--18 Jul 2020.
\newblock URL \url{https://proceedings.mlr.press/v119/yang20j.html}.

\bibitem[Yang et~al.(2021)Yang, Bai, and Mei]{pmlr-v139-yang21a}
Zitong Yang, Yu~Bai, and Song Mei.
\newblock Exact gap between generalization error and uniform convergence in
  random feature models.
\newblock In Marina Meila and Tong Zhang, editors, \emph{Proceedings of the
  38th International Conference on Machine Learning}, volume 139 of
  \emph{Proceedings of Machine Learning Research}, pages 11704--11715. PMLR,
  18--24 Jul 2021.
\newblock URL \url{https://proceedings.mlr.press/v139/yang21a.html}.

\bibitem[Zhang et~al.(2017)Zhang, Bengio, Hardt, Recht, and
  Vinyals]{Zhang:2017}
Chiyuan Zhang, Samy Bengio, Moritz Hardt, Benjamin Recht, and Oriol Vinyals.
\newblock Understanding deep learning requires rethinking generalization.
\newblock In \emph{5th International Conference on Learning Representations,
  {ICLR} 2017, Toulon, France, April 24-26, 2017, Conference Track
  Proceedings}. OpenReview.net, 2017.

\bibitem[Zhang et~al.(2006)Zhang, Berg, Maire, and Malik]{knnsvm}
Hao Zhang, Alexander~C. Berg, Michael Maire, and Jitendra Malik.
\newblock Svm-knn: Discriminative nearest neighbor classification for visual
  category recognition.
\newblock In \emph{CVPR (2)}, pages 2126--2136, 2006.
\newblock URL \url{https://doi.org/10.1109/CVPR.2006.301}.

\end{thebibliography}

\newpage
\appendix
\section{Some concentration results for uniform random variables}
In this section, we state some concentration results that are useful for the theoretical analysis in Section \ref{sec:thy}.
Let $\bxt, \bx_1, \dots, \bx_n\overset{\rm i.i.d.}{\sim}$Unif$(\cB_{\bzero, \sqrt{d+2}})$ be i.i.d. samples from the uniform distribution over the Euclidean norm ball of radius $\sqrt{d+2}$ in $\R^d$.
Let
\begin{equation}
  Z_n = \min_{\bx\in \{\bx_1, \bx_2, \dots, \bx_n\}}\|\bxt-\bx\|^2.
\end{equation}
If $n=1$, $\E Z_1$ is the sum of the variance of each coordinate of Unif$(\cB_{\bzero, \sqrt{d+2}})$.
Therefore, $\E Z_n$ provides a generalized measure of concentration.
Intuitively, $\E Z_n \rightarrow 0$ as $n\rightarrow \infty$.
The proposition below provides a upper bound on the rate of convergence.

\begin{lemma}[Nearest Neighbor concentration]
\label{lem:nearestneighbor}
Given the assumptions above
\begin{equation}
\E Z_n \lesssim  d^2 \left[\frac{\log\left(n^{1/d}\right)}{n}\right]^{1/d},
\end{equation}
where $\lesssim $ means inequality up to an universal constant independent of $d$ and $n$.
\end{lemma}

\begin{proof}
Define
\begin{equation}
\begin{aligned}
\cE_1 &= \{Z_n\leq \delta^2\},\\
\cE_2 &= \{\delta\leq \sqrt{d+2}-\|\bxt\|\}.
\end{aligned}
\end{equation}

We will compute two probabilities $\P(\cE_1|\cE_2)$ and $\P(\cE_2)$ that will be useful latter.
\begin{equation}
\begin{aligned}
\P(\cE_1^c|\cE_2) &= \P(Z_n\geq\delta^2 | \cE_2) = \P(\|\bxt-\bx_i\|\geq\delta,~\forall i|\cE_2),\\
&= \E_{\bxt} \P(\|\bxt-\bx_i\|\geq\delta|\cE_2,\bxt)^n = \E_{\bxt} (1-\P(\|\bxt-\bx_i\|\leq\delta|\cE_2, \bxt))^n,\\
&= \E_{\bxt} \left[1- \frac{\text{Vol}(\cB_{\bxt, \delta})}{\text{Vol}(\cB_{\bzero, \sqrt{d+2}})} \right]^n = \left[1 - \left(\frac{\delta}{\sqrt{d+2}}\right)^d\right]^n,\\
&\leq \exp\left[-n\left(\frac{\delta}{\sqrt{d+2}}\right)^d\right].
\end{aligned}
\end{equation}

Next, we compute $\P(\cE_2)$
\begin{equation}
\begin{aligned}
\P(\cE_2) = \P(\|\bxt\|\leq \sqrt{d+2}-\delta) = \left(\frac{\sqrt{d+2}-\delta}{\sqrt{d+2}}\right)^d = \left(1-\frac{\delta}{\sqrt{d+2}}\right)^d.
\end{aligned}
\end{equation}

We use $\cE_1$ and $\cE_2$ to compute the following upper bound
\begin{equation}
\begin{aligned}
\E Z_n &= \E(Z_n | \cE_1\cap\cE_2) \P(\cE_1\cap\cE_2) + \E (Z_n|(\cE_1\cap\cE_2)^c) P((\cE_1\cap\cE_2)^c), \\
&\leq \delta^2 + (2\sqrt{d+2})^2 \left(1 - \P(\cE_1\cap\cE_2)\right), \\
&= \delta^2 + 4(d+2) \left[1-\P(\cE_1|\cE_2)\P(\cE_2)\right].
\end{aligned} 
\end{equation}

To find an upper bound for $\E Z_n$, we need to find an upper bound for $1-\P(\cE_1|\cE_2)\P(\cE_2)$.
\begin{equation}
\begin{aligned}
1-\P(\cE_1|\cE_2)\P(\cE_2) &= 1-\left[1-\P(\cE_1^c|\cE_2)\right]\P(\cE_2), \\
&= 1 -\P(\cE_2) + \P(\cE_1^c|\cE_2)\P(\cE_2),\\
&\leq 1-\P(\cE_2) + \P(\cE_1^c|\cE_2).
\end{aligned}
\end{equation}
Now choose $\delta=\sqrt{d+2}n^{-1/d}\left[\log\left(n^{1/d}\right)\right]^{1/d}$.
\begin{equation}
\P(\cE_1^c|\cE_2)\leq \exp\left[-n\left(\frac{\delta}{\sqrt{d+2}}\right)^d\right] = \exp\left[-nn^{-1}\log\left(n^{1/d}\right)\right] = n^{-1/d},
\end{equation}

and
\begin{equation}
\P(\cE_2) = \left(1-\frac{\delta}{\sqrt{d+2}}\right)^d\geq 1-d\frac{\delta}{\sqrt{d+2}} = 1-dn^{-1/d}\left[\log\left(n^{1/d}\right)\right]^{1/d}.
\end{equation}
Thus
\begin{equation}
1-\P(\cE_1|\cE_2)\P(\cE_2) \leq 1 - 1 + dn^{-1/d}\left[\log\left(n^{1/d}\right)\right]^{1/d} + n^{-1/d} \lesssim dn^{-1/d}\left[\log\left(n^{1/d}\right)\right]^{1/d}.
\end{equation}
Combining everything together, we get
\begin{equation}
\begin{aligned}
\E Z_n &\leq (d+2) n^{-2/d} \left[\log\left(n^{1/d}\right)\right]^{2/d} + 4(d+2)\times dn^{-1/d}\left[\log\left(n^{1/d}\right)\right]^{1/d},\\
&\lesssim d^2 n^{-1/d}\left[\log\left(n^{1/d}\right)\right]^{1/d},\\
&= d^2 \left[\frac{\log\left(n^{1/d}\right)}{n}\right]^{1/d}.
\end{aligned} 
\end{equation}
This completes the proof.
\end{proof}

\begin{proposition}[\cite{wainwright2019high} Corollary 6.20]
    \label{prop:matcon}
    Let $\bx_i\overset{\rm i.i.d.}{\sim}$Unif$~(\cB_{\bzero, \sqrt{d+2}})$ for $i=1, \dots, n$ be uniformly distributed over a ball of radius $B$ in $\R^d$ centered at $\bzero$.
    Let
    \[
        \bSigma_n = \frac{1}{n} \sum_{i=1}^n \bx_i\bx_i^\sT
    \]
    be the sample covariance matrix. Then
    \[
        \P(\|\bSigma_n - \bI\|_{\rm op} > \eps) \leq 2d\exp\left[ - \frac{n\eps^2}{2(d+2)(1+\eps)}\right].
    \]
\end{proposition}

\newpage
\section{Proof of Theorem \ref{thm:test-loss-ltm}}
\label{sec:lin-proof}
In this section, we present the proof of Theorem \ref{thm:test-loss-ltm}.
In Section \ref{sec:lin-decomp-detail}, we provide the detail of the decomposition of the risk into $T_1$ and $T_2$.
Then in Section \ref{sec:upb-t1} we compute an upper bound for $T_1$, and compute an upper bound for $T_2$ in Section \ref{sec:upd-t2}.
Finally, we combine everything together in Section \ref{sec:final-thm} and completes the proof.

\subsection{Decomposition of the test risk}
\label{sec:lin-decomp-detail}
\begin{equation}
\begin{aligned}
&~~~~~\E  \left[f^{\alg}(\bxt)-f_\star(\bxt)\right]^2 
    = \E  \left[f_n(\bxt)+r_n(\bxt)-f_\star(\bxt)\right]^2,\\
    &= \E  \left[ f_n(\bxt) - f_\star(\bxt)
                      -f_n(\bxnn) + f_\star(\bxnn)\right]^2,\\
    &= \E  \left[ f_n(\bxt)
                      -f_\infty(\bxt) + f_\infty(\bxt)
                      -f_\star(\bxt)
                      -f_n(\bxnn)
                      +f_\infty(\bxnn) - f_\infty(\bxnn)
                      +f_\star(\bxnn)
                 \right]^2,\\
    &\leq 3\times [ \underbrace{\E  (f_n(\bxt)-f_\infty(\bxt))^2
                        +\E  (f_n(\bxnn)-f_\infty(\bxnn))^2}_{T_1} 
                    +\underbrace{\E  (f_\infty(\bxt)
                       -f_\star(\bxt)- f_\infty(\bxnn)
                       +f_\star(\bxnn))^2}_{T_2}
                  ],
\end{aligned}
\end{equation}
where in the last inequality, we used the fact that $(a+b+c)^2<3(a^2+b^2+c^2)$ for any $a, b, c\in\R$.

\ActivateWarningFilters[pdftoc]
\subsection{Upper bound on $T_1$.}
\DeactivateWarningFilters[pdftoc]

\label{sec:upb-t1}
\noindent Since $\P_\bx = \text{Unif}(\cB_{\bzero, B}~ )$, we apply the bound $\|\bxt\|, \|\bxnn\|\leq B$ to obtain
\begin{equation}
\begin{aligned}
T_1
&= \E [f_n(\bxt) - f_\infty (\bxt)]^2 + \E [f_n(\bxnn) - f_\infty(\bxnn)]^2, \\
&= \E \<\btheta_n-\btheta_\infty, \bxt\>^2 + \E \<\btheta_n - \btheta_\infty, \bxnn\>^2, \\
&\leq \E\|\btheta_n-\btheta_\infty\|^2 \|\bxt\|^2 + \E \|\btheta_n-\btheta_\infty\|^2 \|\bxnn\|^2,\\
&\leq 2B^2\E\|\btheta_n - \btheta_\infty\|^2.
\end{aligned}
\end{equation}

As $n$ gets large, the empirical covariance matrix $\bSigma_n = \bX^\sT \bX/n$ is concentrated around its mean $\bI$.
Let $\bDelta_n = \bI - \bSigma_n$ denote this deviation.
For some $\eps \in (0, 1)$, define the following ``good event'' over the randomness in $\bSigma_n$
\begin{equation}
\cA = \{\|\bDelta_n\|_{\rm op} <\eps\},
\end{equation}

where $\|\bDelta_n\|_{\rm op}$ denotes the operator norm of the deviation matrix.
The high level idea of the proof is to condition on the event $\cA$ and deduce and upper bound of $\|\btheta_n-\btheta_\infty\|$ in terms of $\eps$.
Then, we use the fact that $\cA$ happens with high probability.

Recall that $\btheta_\infty = L \btheta_\star$, and 
\begin{equation}
\btheta_n = \underset{\|\btheta\|\leq L}{\operatorname{argmin}} \, \frac{1}{n} \|\bX\btheta-\by\|^2.
\end{equation}

Since $\by = \bX\btheta_\star$ by definition,
the Lagrangian of the convex program above is
\begin{equation}
\mathcal{L}(\btheta, \lambda) = \frac{1}{n} \|\bX\btheta-\bX\btheta_\star\|^2 + \lambda (\|\btheta\|^2 - L).
\end{equation}

The KKT condition suggests that the primal-dual optimal pair $(\btheta_n, \lambda_n)$ is given by
\begin{equation}
\begin{aligned}
\|\btheta_n\|&\leq L ,\\
\lambda_n    &\geq 0 ,\\
\lambda_n(\|\btheta_n\|-L)&=0 ,\\
\end{aligned}
\end{equation}

and at optimality
\begin{equation}
\begin{aligned}
\grad_\btheta \mathcal{L}(\btheta_n, \lambda_n) = 0 &\iff \frac{2}{n}\bX^\sT\bX(\btheta-\btheta_\star) + 2\lambda_n \btheta = 0,  \\
&\iff \btheta_n = (\bSigma_n+\lambda_n \bI)^{-1}\bSigma_n\btheta_\star.\\
\end{aligned}
\end{equation}

The complementary slackness condition 
$\lambda_n(\|\btheta_n\|-L)=0$ suggests that either $\lambda_n=0$ or $\|\btheta_n\|=L$.
But if $\lambda_n = 0$, the stationary condition $\grad_\btheta \mathcal{L}(\btheta, \lambda) = 0$ would suggest that 
$\btheta_n = \bSigma_n^{-1}\bSigma_n\btheta_\star = \btheta_\star \Rightarrow \|\btheta_n\|=1>L,$ a contradiction.
(Note that here $\bSigma_n$ is invertible condition on the event $\cA$.)
Therefore, we must have $\|\btheta_n\|=L$. 
As a result, the primal and dual pair $(\btheta_n, \lambda_n)$ is determined by the system of equations
\begin{equation}
\begin{cases}
\btheta_n &= (\bSigma_n + \lambda_n \bI)^{-1} \bSigma_n \btheta_\star, \\
\|\btheta_n\| &= L, \\
\lambda_n &> 0.
\end{cases}
\end{equation}
Next, we proceed to compute the deviation $\|\btheta_n - \btheta_\infty\|$.
\begin{equation}
\begin{aligned}
\btheta_n
&= \left[ (\lambda_n+1)\bI - \bDelta_n\right]^{-1} \bSigma_n \btheta_\star, \\
&= (\lambda_n+1)^{-1} \left[ \bI - \frac{\bDelta_n}{\lambda_n+1}\right]^{-1}\bSigma_n \btheta_\star, \\
&= (\lambda_n+1)^{-1} \left[ \bI + \sum_{k=1}^{\infty} \frac{\bDelta_n^k}{(\lambda_n+1)^k}\right](\bI-\bDelta_n) \btheta_\star,\\
&=(\lambda_n+1)^{-1} \left[ \bI + \sum_{k=1}^\infty \frac{\bDelta_n^k}{(\lambda_n+1)^k} - \bDelta_n - \sum_{k=1}^\infty \frac{\bDelta_n^{k+1}}{(\lambda_n+1)^k} \right]\btheta_\star,\\
&=(\lambda_n+1)^{-1}\btheta_\star + (\lambda_n+1)^{-1} \bDelta_n \left[\sum_{k=1}^\infty \frac{\bDelta_n^{k-1}}{(\lambda_n+1)^k} - \bI - \sum_{k=1}^\infty \frac{\bDelta_n^{k}}{(\lambda_n+1)^k} \right]\btheta_\star,\\
&= (\lambda_n+1)^{-1}\btheta_\star + (\lambda_n+1)^{-1} \bDelta_n\left[\sum_{k=1}^{\infty} \frac{\bDelta_n^{k-1}-\bDelta_n^k}{(\lambda_n+1)^k} - \bI\right]\btheta_\star.
\end{aligned}
\end{equation}

Define
\begin{equation}
\bD_n = \bDelta_n\left[\sum_{k=1}^{\infty} \frac{\bDelta_n^{k-1}-\bDelta_n^k}{(\lambda_n+1)^k} - \bI\right].
\end{equation}

Then $\btheta_n = (\lambda_n+1)^{-1}\btheta_\star + (\lambda_n+1)^{-1}\bD_n\btheta_\star$, and
\begin{equation}
\begin{aligned}
\|\bD_n\| &\leq \|\bDelta_n\| \left[1+\sum_{k=1}^{\infty} \frac{\|\bDelta_n\|^{k-1}+\|\bDelta_n\|^k}{(\lambda_n+1)^k}\right] ,\\
&\leq \eps \left[1+2(1+\lambda_n)^{-1}\sum_{k=1}^\infty \left(\frac{\eps}{1+\lambda_n}\right)^k\right],\\
&= \eps\left(1 + \frac{2}{1+\lambda_n} \frac{1}{1 - \frac{\eps}{1+\lambda_n}}\right) \leq 3\eps.
\end{aligned}
\end{equation}

Therefore
\begin{equation}
\begin{aligned}
&L=\|\btheta_n\|^2 = (\lambda_n+1)^{-2} + (\lambda_n+1)^{-2} \btheta_\star^\sT \bD_n^\sT \bD_n \btheta_\star + 2(\lambda_n+1)^{-2} \btheta_\star \bD_n \btheta_\star, \\
\Rightarrow & (\lambda_n+1)^2L^2 = 1 + \delta_n,~ \delta_n = \btheta_\star^\sT\bD_n^\sT\bD_n\btheta_\star + 2\btheta_\star^\sT\bD_n\btheta_\star.
\end{aligned}
\end{equation}

We can obtain the following bound for $\delta_n$:
\begin{equation}
|\delta_n| \leq \|\btheta_\star\|^2 \|\bD_n\|^2 + 2\|\btheta_\star\|^2 \|\bD_n\| \leq 9\eps^2 + 6\eps \leq 15\eps.
\end{equation}

Since $1-\delta_n/2\leq \sqrt{1+\delta_n}\leq 1+\delta_n/2$, and $|\delta_n|\leq 15\eps$, we obtain
\begin{equation}
|(\lambda_n+1)L-1|\leq \frac{15\eps}{2} \Rightarrow \left|L-(\lambda_n+1)^{-1}\right|\leq \frac{15\eps}{2}(\lambda_n+1)^{-1} \leq \frac{15\eps}{2},
\end{equation}

where the last inequality follows as we have $\lambda_n > 0$.
Finally,
\begin{equation}
\begin{aligned}
\btheta_n -\btheta_\infty~~~ &= (\lambda_n+1)^{-1}\btheta_\star - L \btheta_\star + (\lambda_n+1)^{-1}\bD_n\btheta_\star,\\
\Rightarrow \|\btheta_n-\btheta_\infty\|^2 &= [(1+\lambda_n)^{-1}-L]^2 + (1+\lambda_n)^{-2} \btheta_\star \bD_n^\sT\bD_n\btheta_\star + 2(\lambda_n+1)^{-1}[(1+\lambda_n)^{-1}-L] \btheta_\star \bD_n \btheta_\star,\\
&\leq 64\eps^2 + 9\eps^2 + 45\eps^2 = 118\eps^2, \\
\Rightarrow \|\btheta_n - \btheta_\infty\|^2 &\lesssim \eps^2.
\end{aligned}
\end{equation}

Combine the above result with Proposition \ref{prop:matcon}, we get that

\begin{equation}
\begin{aligned}
\E\|\btheta_n - \btheta_\infty\|^2 &= \E (\|\btheta_n - \btheta_\infty\|^2|\cA)\P(\cA) + \E(\|\btheta_n - \btheta_\infty\|^2|\cA^c)\P(\cA^c), \\
&\leq \eps^2 + 4L^2 \times 4d\exp\left[ -\frac{n\eps^2}{2(d+2)(1+\eps)} \right], \\
\end{aligned}
\end{equation}

If we choose $\eps = n^{-1/3}$, we get
\begin{equation}
\E\|\btheta_n - \btheta_\infty\|^2 \lesssim dL^2n^{-2/3},
\end{equation}

which implies that
\begin{equation}
T_1 \lesssim d^2L^2n^{-2/3}.
\end{equation}

\ActivateWarningFilters[pdftoc]
\subsection{Upper bound on $T_2$.}
\DeactivateWarningFilters[pdftoc]

\label{sec:upd-t2}
Plugging in the formula for $f_\perp(\bxt) = f_\star(\bxt)-f_\infty(\bxt) = \<\bxt, \btheta_\perp\>$, we get
\begin{equation}
\begin{aligned}
  T_2 &= \E [f_\perp(\bxnn)-f_\perp(\bxt)]^2,\\
  &= \E\<\btheta_\perp, \bxnn-\bxt\>^2,\\
  &\leq (1-L)^2\|\btheta_\star\|^2 \E \|\bxt-\bxnn\|^2,\\
  &= (1-L)^2 \E \|\bxt-\bxnn\|^2,
\end{aligned}
\end{equation}

where in the last inequality, we used the relation that $\btheta_\perp = (1-L)\btheta_\star$.
Proposition \ref{lem:nearestneighbor} suggests that
\begin{equation}
\E \|\bxt-\bxnn\|^2 \lesssim d^2 \left[\frac{\log\left(n^{1/d}\right)}{n}\right]^{1/d},
\end{equation}

which implies 
\begin{equation}
T_2 \lesssim d^2(1-L)^2  \left[\frac{\log\left(n^{1/d}\right)}{n}\right]^{1/d}.
\end{equation}

\begin{remark}[Comparison with pure nearest neighbor and ERM]
If we rely solely on nearest neighbor method, the prediction error is 
\begin{equation}
  \E[f_\star(\bxt)-f_\star(\bxnn)]^2 = \E \<\bxt-\bxnn, \btheta_\star\>^2 \leq \E \|\bxt-\bxnn\|^2.
\end{equation}

On the other hand, if we solely rely on ERM, even with infinite sample, we get
\begin{equation}
  \E[f_\star(\bxt)-f_\infty(\bxt)]^2 = \E \<\bxt, \btheta_\star-\btheta_\infty\>^2 \leq (1-L)^2\E \|\bxt\|^2.
\end{equation}

We can see from the upper bound that $\alg$ takes advantage of both 
\begin{itemize}
    \item Projecting $f_\star$ onto $f_\infty$, so that the dependence on the prediction function is reduced from $1$ to $(1-L)^2$.
    \item Memorizing the residuals using nearest neighbor, so that the variance is reduced from $\E\|\bxt\|^2$ to $\E\|\bxnn-\bxt\|^2$.
\end{itemize}
\end{remark}

\ActivateWarningFilters[pdftoc]
\subsection{Test loss for $\alg$.}
\DeactivateWarningFilters[pdftoc]
\label{sec:final-thm}

If we combine the previous two parts together, we get
\begin{equation}
\begin{aligned}
\E \left[\hat{f}(\bxt) - f_\star(\bxt)\right]^2 &\lesssim d^2L^2n^{-2/3} + d^2(1-L)^2  \left[\frac{\log\left(n^{1/d}\right)}{n}\right]^{1/d}.
\end{aligned}
\end{equation}

\noindent This completes the proof of Theorem~\ref{thm:test-loss-ltm}.

\newpage

\section{Additional CIFAR100 Results}
\label{sec:add-cifar}
This section includes additional experiment results on applying ResMem to CIFAR100 dataset.
\subsection{Additional robustness results}
In addition to the results already presented in Section \ref{sec:exp-vision}, we also evaluate ResMem performance for each architecture in CIFAR-ResNet\{8, 14, 20, 32, 44, 56\} and each subset (10\%, 20\%, ..., 100\%) of CIFAR100 training data.
We use the same training hyperparameter and the ResMem hyperparameter as described in Section \ref{sec:exp-vision}.
Generally, we see that ResMem yields larger improvement over the baseline DeepNet when the network is small and dataset is large.
\begin{figure*}[ht]
    \centering
    \subfigure[\label{fig:cifar-sample-resnet8-appendix}CIFAR-ResNet-8]{
    \includegraphics[width=.24\textwidth]{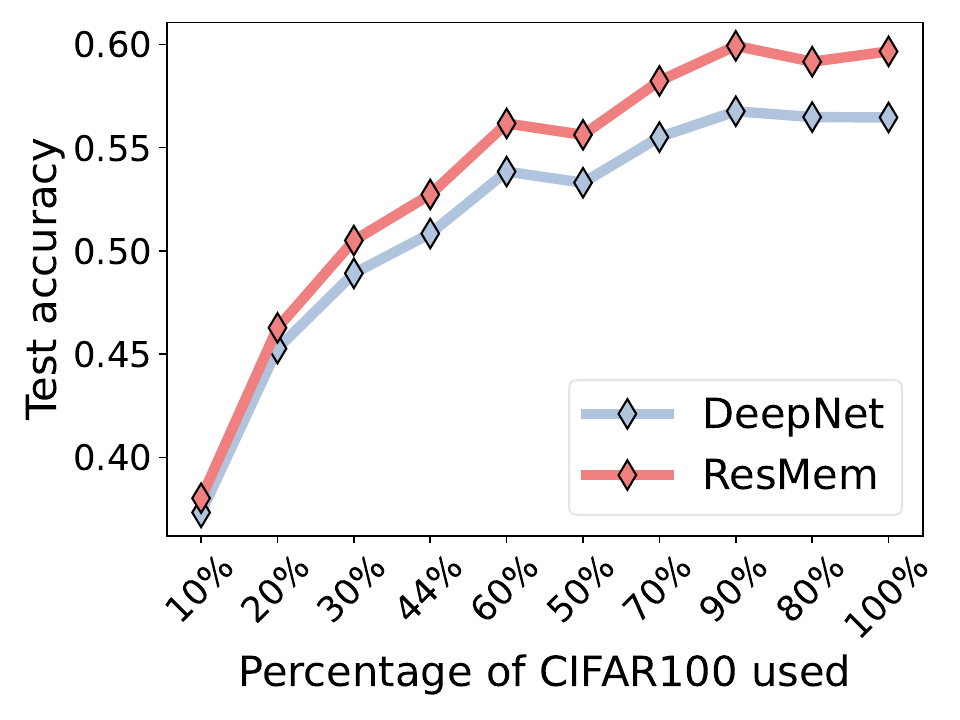}
    \includegraphics[width=.24\textwidth]{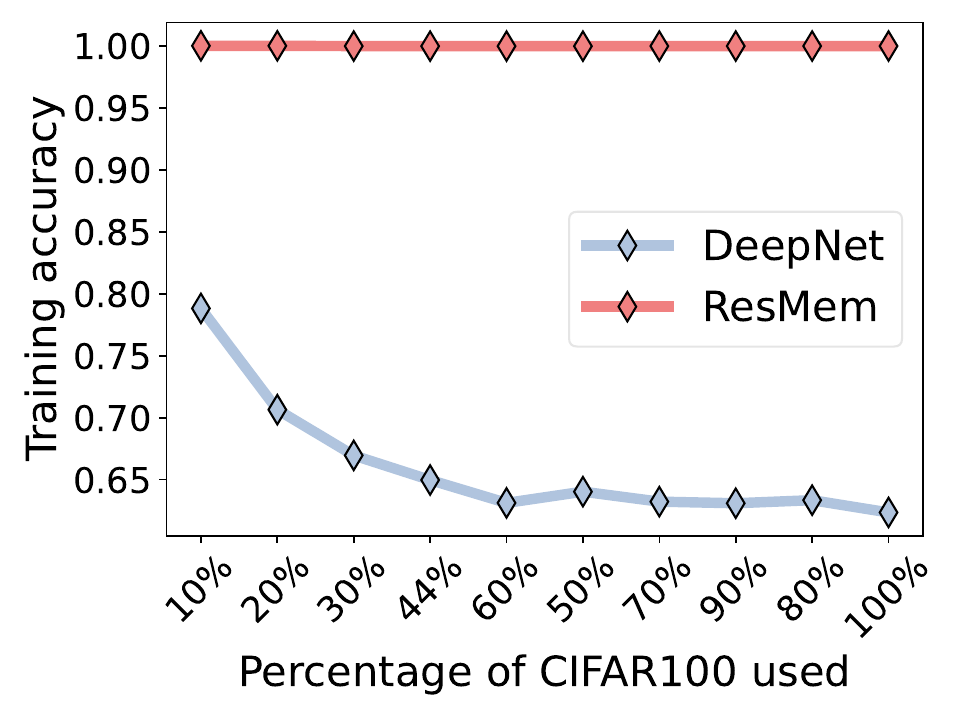}}
    \subfigure[\label{fig:cifar-sample-resnet14}CIFAR-ResNet-14]{
    \includegraphics[width=.24\textwidth]{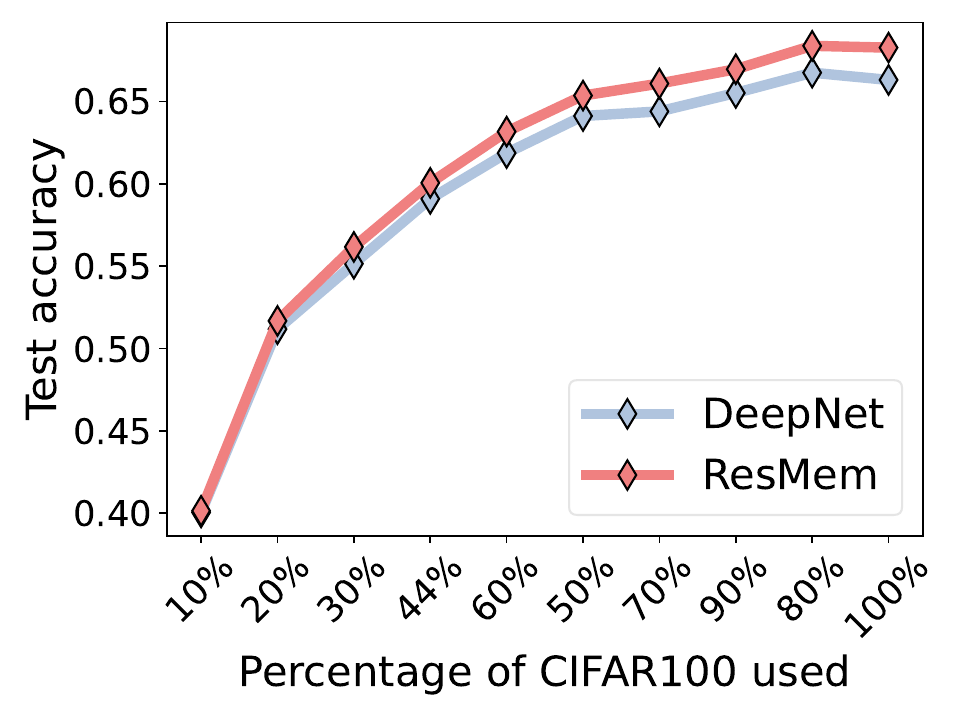}
    \includegraphics[width=.24\textwidth]{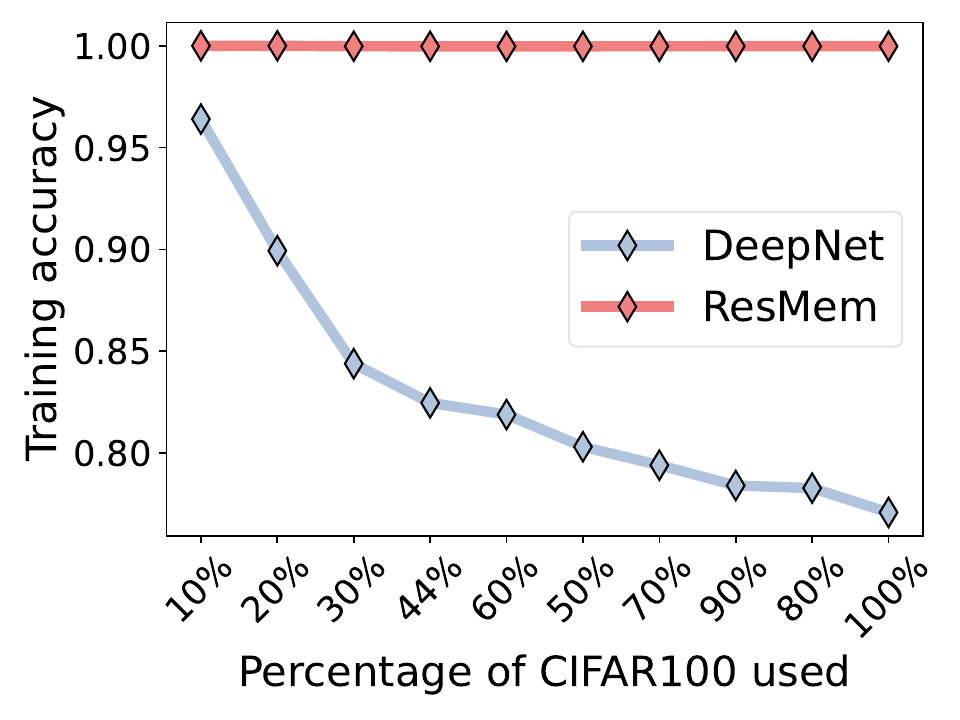}}
    \subfigure[\label{fig:cifar-sample-resnet20}CIFAR-ResNet-20]{
    \includegraphics[width=.24\textwidth]{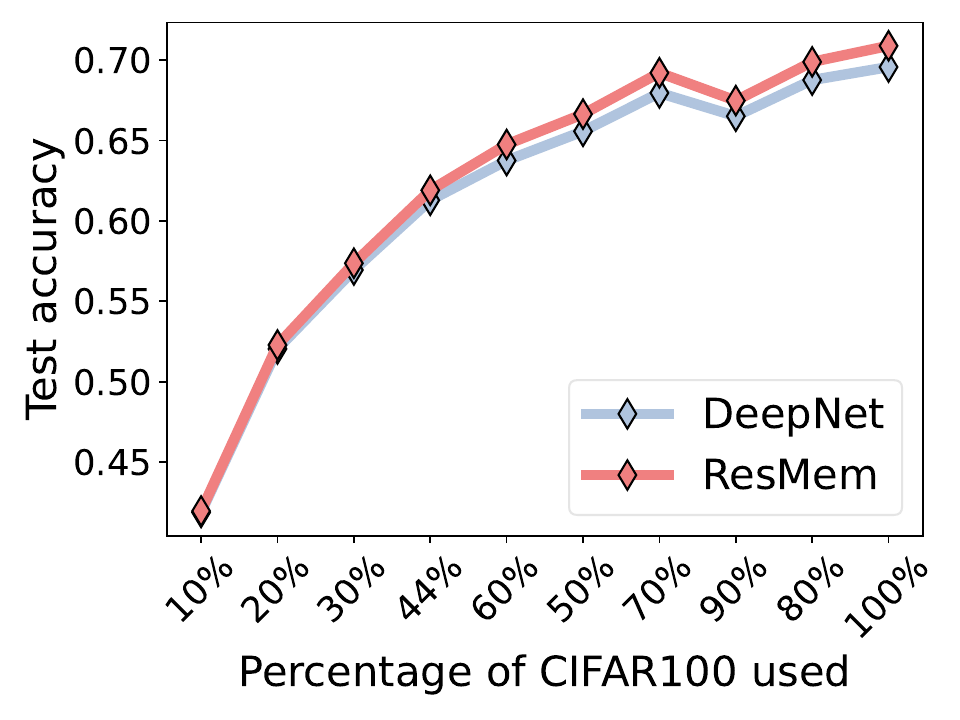}
    \includegraphics[width=.24\textwidth]{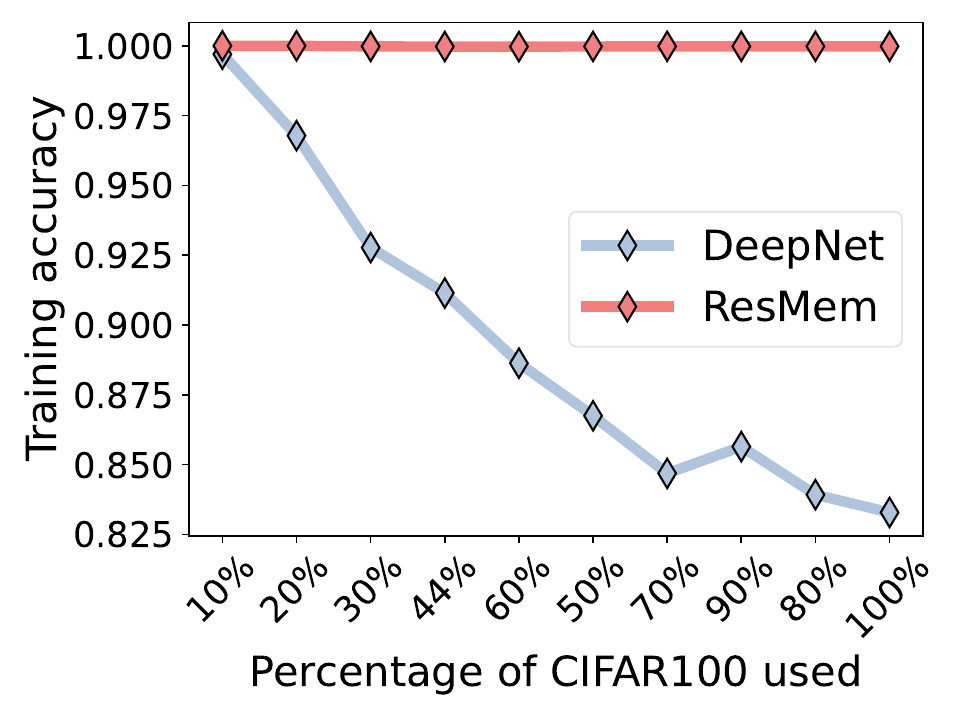}}
    \subfigure[\label{fig:cifar-sample-resnet32}CIFAR-ResNet-32]{
    \includegraphics[width=.24\textwidth]{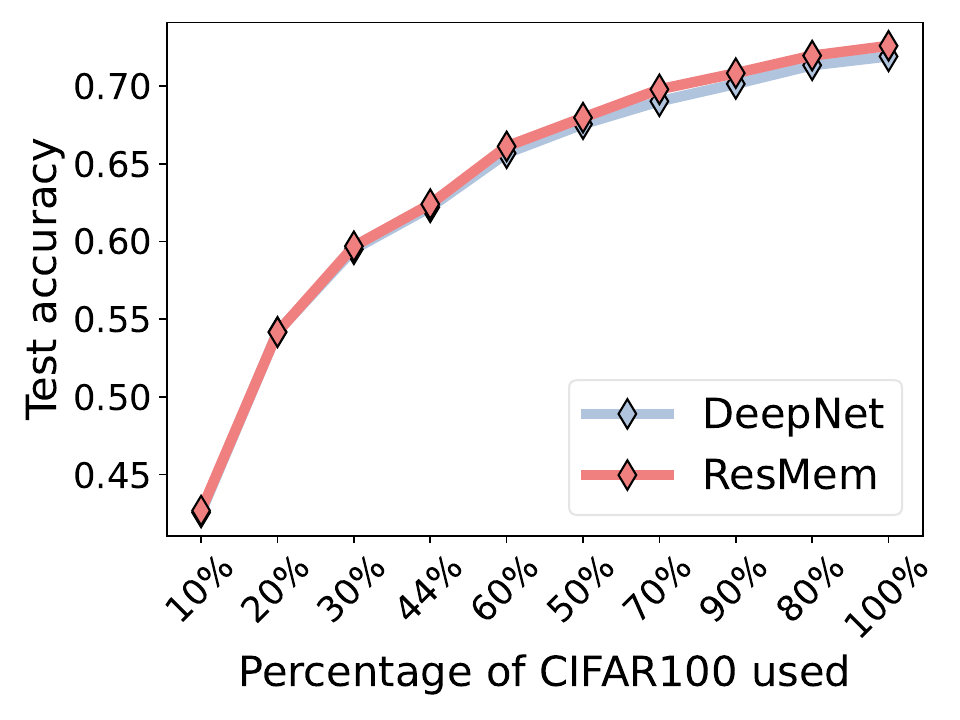}
    \includegraphics[width=.24\textwidth]{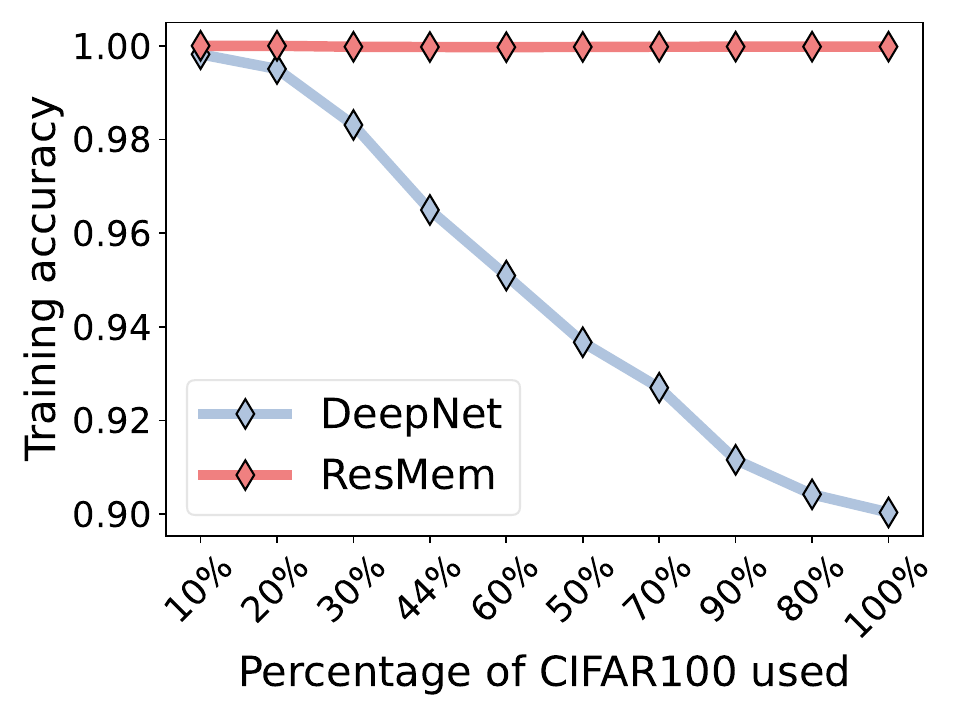}}
    \subfigure[\label{fig:cifar-sample-resnet44}CIFAR-ResNet-44]{
    \includegraphics[width=.24\textwidth]{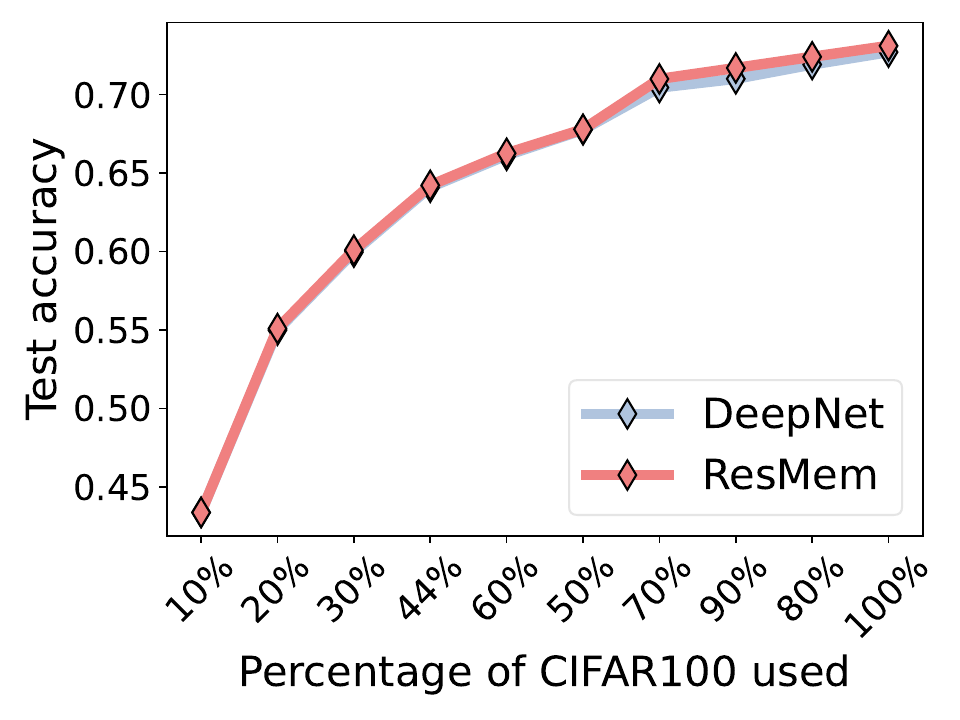}
    \includegraphics[width=.24\textwidth]{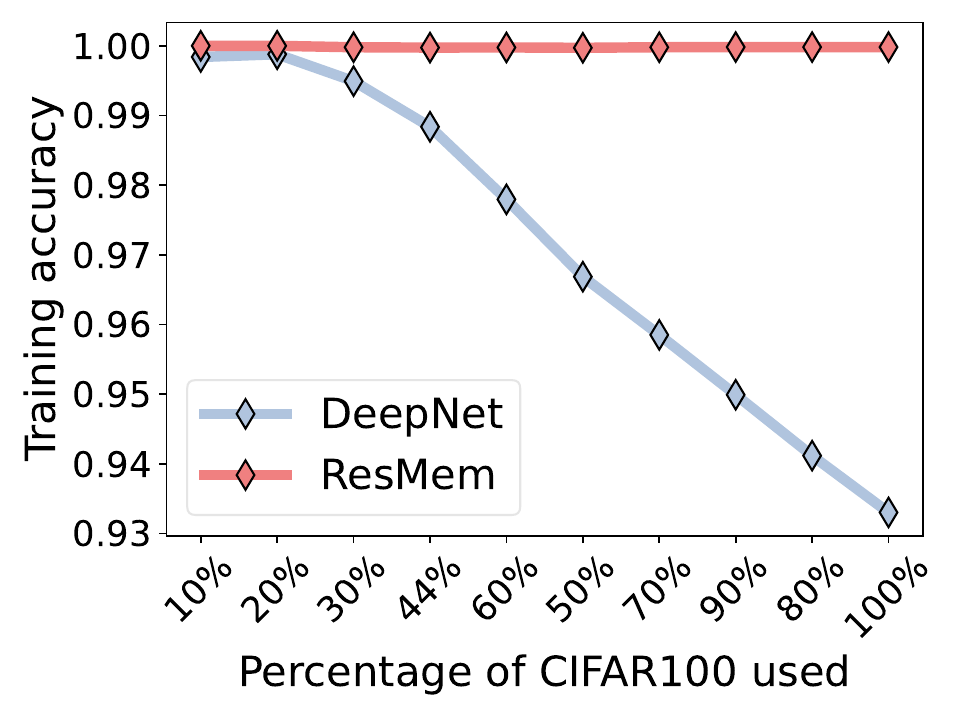}}
    \subfigure[\label{fig:cifar-sample-resnet56}CIFAR-ResNet-56]{
    \includegraphics[width=.24\textwidth]{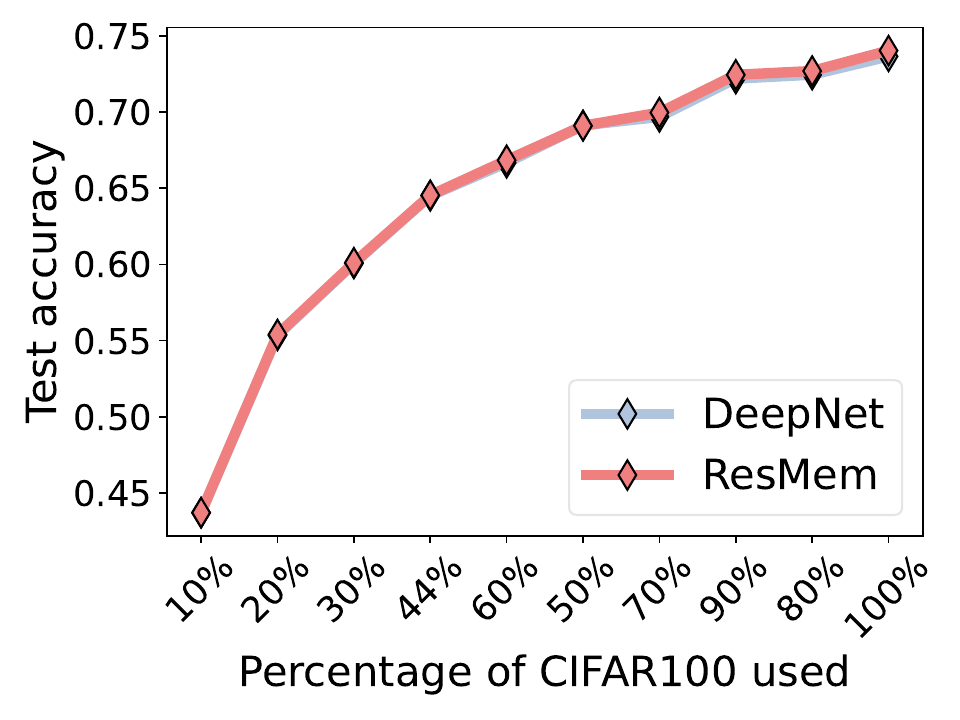}
    \includegraphics[width=.24\textwidth]{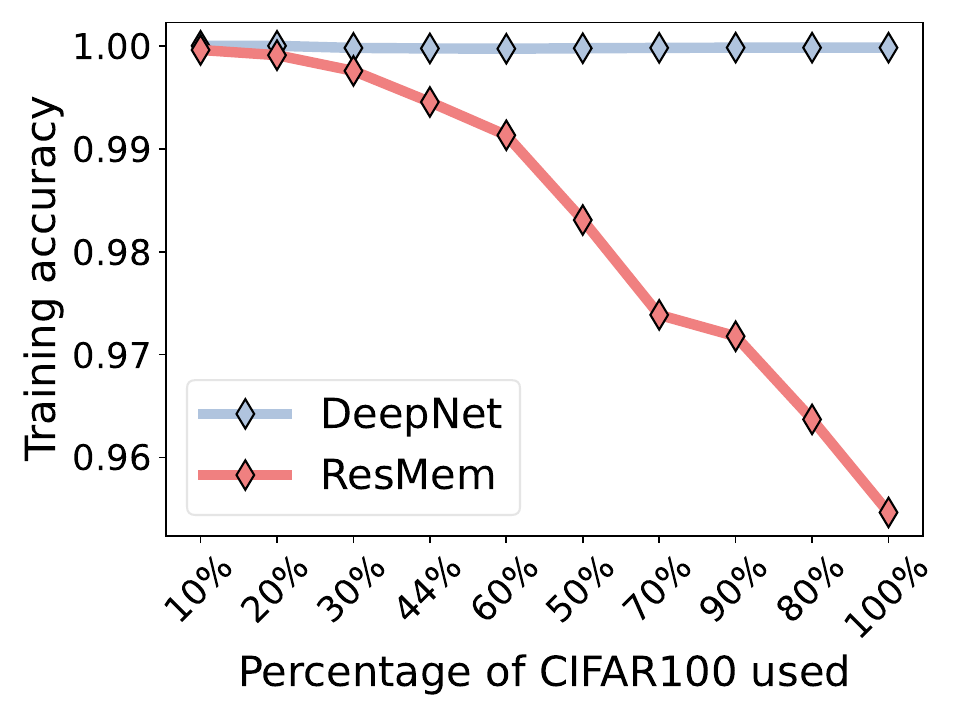}}
    \caption{Test(left)/Training (right) accuracy for different sample sizes.}
    \label{fig:cifar-all}
\end{figure*}

\subsection{Sensitivity analysis for CIFAR100}
\label{sec:sensitivity-cifar}
\begin{figure*}[ht]
    \centering
    \subfigure[\label{fig:cifar-k}\# of neighbours $k$.]{\includegraphics[width=.32\textwidth]{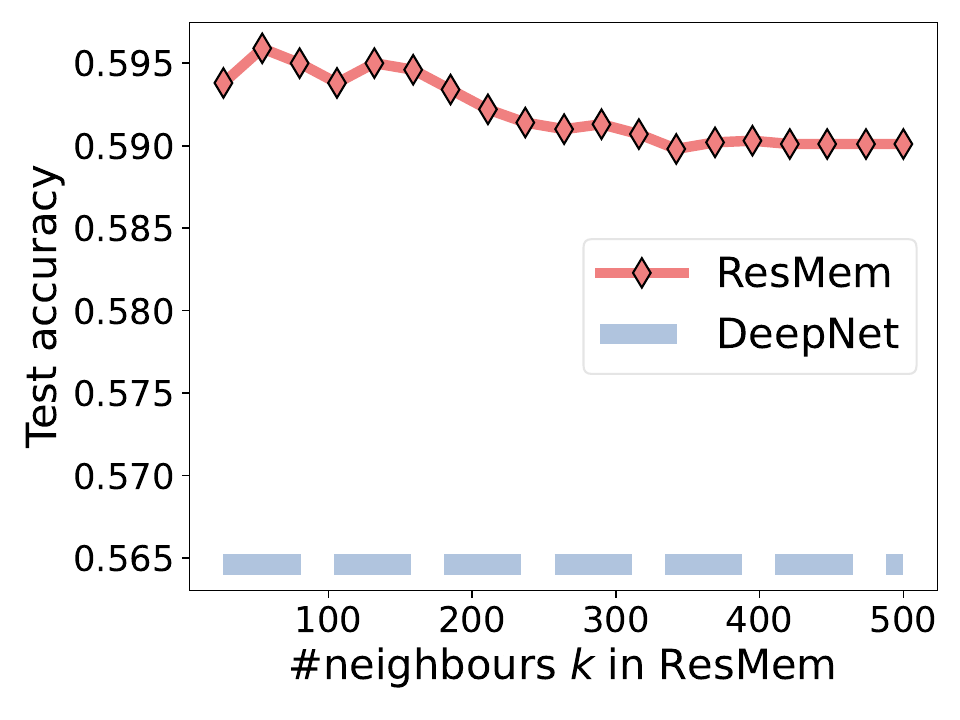}}
    \subfigure[\label{fig:cifar-sigma}Radius parameter $\sigma$.]{\includegraphics[width=.32\textwidth]{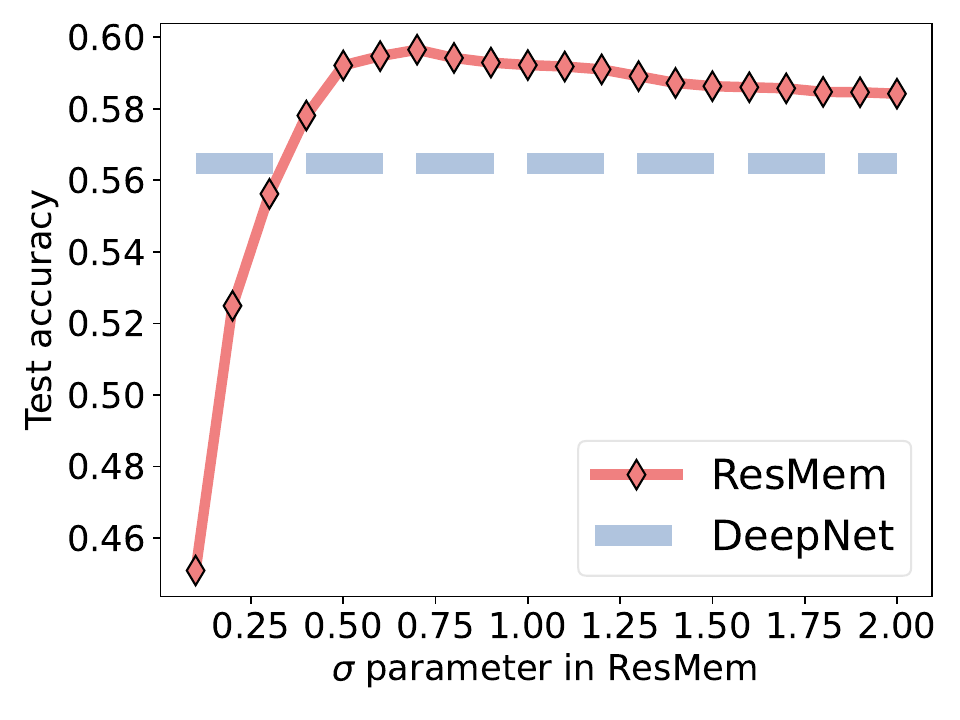}}
    \subfigure[\label{fig:cifar-temp}Temperature $T$.]{\includegraphics[width=.32\textwidth]{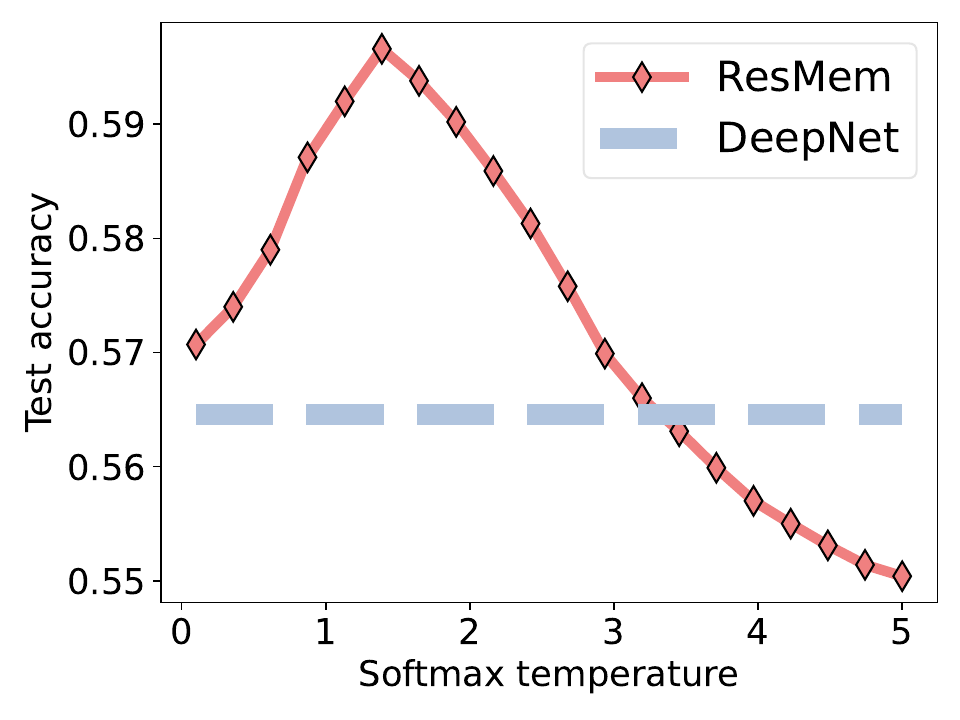}}
    \caption{
    Sensitivity analysis of ResMem hyperparameters.
    The $y$-axis represents the CIFAR100 test accuracies,
    and the
    $x$-axis represents the sweeping of respective hyperparameters.}
    \label{fig:robustness-hyper}
\end{figure*}

\paragraph{Varying locality parameter $k$ and $\sigma$.}
We vary the number of neighbours from $k=27$ to $k=500$.
We find that  ResMem test accuracy is relatively stable across the choice of the number of neighbours (cf.~Figure \ref{fig:cifar-k}).
The trend of the curve suggests that as $k\rightarrow\infty$, the ResMem test accuracy seems to be converging to a constant level.
For $\sigma$, we explored different values of $\sigma \in (0.1, 2.0)$. We observe that the test accuracy has a unimodal shape as a function of $\sigma$, suggesting that there is an optimal choice of $\sigma$ (cf.~Figure \ref{fig:cifar-sigma}).

\paragraph{Varying temperature $T$ and connection to distillation.}
We tried $T=0.1$ to $T=5$, and also identified an unimodal shape for the test accuracy (Figure \ref{fig:cifar-temp}).
The fact that we can use different temperatures for (a) training the network and (b) constructing the $k$-NN predictor reminds us of the well-established knowledge distillation  procedure \citep{distillation}.
In knowledge distillation, we first use one model (the teacher network) to generate targets at a higher temperature, and then train a second model (the student network) using the \emph{combination} of the true labels and the output of the first network.

ResMem operates in a reversed direction:
Here we have a second model ($\knn$) that learns the \emph{difference} between true labels and the output of the first model.
In both cases, we can tune the temperature of the first model to control how much information is retained.
This connection offers an alternative perspective that regards ResMem as a ``dual procedure'' to knowledge distillation.

\section{ResMem on ImageNet}
\label{sec:imagenet}
This section includes additional experiment results on applying ResMem to ImageNet dataset.

\paragraph{ImageNet.}
In addition to CIFAR100, we also evaluate the performance of ResMem on ImageNet~\citep{ILSVRC15}. 
We employ a family of pre-trained MobileNet-V2 models~\citep{Sandler:2018}
from Keras\footnote{\url{https://keras.io/api/applications/mobilenet/}},
with varying widths controlled by a multiplier $a$.
For ResMem, we again use the second last layer of DeepNet as a 1280-dimensional embedding of an image and rely on the $\ell_2$ distance between the embeddings for nearest neighbor search (Step 3, Section \ref{sec:background}).
We specify the ResMem parameter of $(k, \sigma, T)$ in the table below.
We repeat the experiment over several MobileNet-V2 architectures, with MobileNet-V2-a0.35 being the smallest model and MobileNet-V2-a1.3 being the largest one.

\begin{table}[ht]
\centering
\caption{Test accuracy for ResMem and baseline deep network for ImageNet data.}
\label{table:imagenet}
\vskip 0.1in
\scalebox{1}{
\begin{tabular}{@{}cccccc@{}} 
\toprule
\multirow{2}{*}{Architecture}  & \multicolumn{3}{c}{ResMem param.} & \multicolumn{2}{c}{Test accuracy} \\
\cmidrule(l){2-4}
\cmidrule(l){5-6}
&  $k$ & $\sigma$ & $T$ &\nn  & \alg \\
\toprule
MobileNet-V2-a0.35 & 10 & 0.6 & 0.4 & 60.2\%  & \textbf{61.2\%} \\
MobileNet-V2-a0.5 & 10 & 0.6 & 0.4 & 65.3\%  & \textbf{66.1\%} \\
MobileNet-V2-a0.75 & 10 & 0.8 & 0.6 & 69.6\%  & \textbf{70.1\%} \\
MobileNet-V2-a1.0 & 20 & 0.4 & 0.4 & 71.3\%  & \textbf{71.8\%} \\
MobileNet-V2-a1.3 & 30 & 0.4 & 0.4 & 74.7\%  & \textbf{75.1\%} \\
\toprule
\end{tabular}
}
\end{table}
We can see that (c.f. Table \ref{table:imagenet}) ResMem boosts the test accuracy by $1\%$ on the smallest model and by $0.4\%$ on the largest model.

\section{Additional details of NLP experiments}
\label{sec:detail-nlp}

The Decoder-Only model used in our experiments is essentially the normal Encoder-Decoder architecture with Encoder and Cross-Attention removed.
We pretrained both the T5-small and T5-base model on C4~\citep{2020t5} dataset with auto-regressive language modeling task for 1,000,000 steps, with dropout rate of 0.1 and batch size of 128.
The learning rate for the first 10,000 steps is fixed to 0.01 and the rest steps follow a square root decay schedule.

During the inference for retrieval key, query embeddings and residuals, we ensured every token has at least 64 preceding context by adopting a sliding window strategy, where a window of 256 token slides from the beginning to the end on each of the articles, with a stride of $256 - 64 = 192$.

For residuals, we only stored the top 128 residuals measured by the absolute magnitude, as the residual vector is as large as T5 vocabulary size (i.e., 32128), and storing all 32128 residuals for each token is too demanding for storage. However, when weight-combining the residuals, we zero filled the missing residuals so that all the residual vectors have 32128 elements.

\section{Comparison with other algorithms}
\label{sec:knnlm}
We mainly compare ResMem against \cite{knnlm}, where the algorithm uses $k$NN to retrive labels directly instead of the residual of the label.
In their algorithm, a key aparameter is $\lambda\in[0, 1]$ which specifeis how much weight to give to the neural network and how much for the $k$NN component.
In the extreme case of $\lambda$=1, their algorithm reduces to using $k$NN to memorize data directly.

For the language modeling task, we use the C4 dataset and T5-large architecture.
As we change the weight \cite[Equation (3)]{knnlm} of the DeepNet component, we find the best performing kNN-LM methods has accuracy 44.88\% which is lower accuracy than the ResMem accuracy 45.55\%.
In particular, we obtain the table below

\begin{table}[h]
\centering
\caption{Test accuracy for kNN-LM (ResMem accuracy 45.55\%)}
\begin{tabular}{lcccccccc}
\toprule
kNN weight $\Lambda$ & 0 & 0.2 & 0.4 & 0.5 & 0.6 & 0.8 & 1 \\
\midrule
kNN-LM accuracy & 44.76\% & 44.88\% & 44.83\% & 44.66\% & 44.27\% & 42.97\% & 40.95\% \\
ResMem acc. - kNN-LM acc. & 0.79\% & 0.67\% & 0.72\% & 0.89\% & 1.28\% & 2.58\% & 4.60\% \\
\bottomrule
\end{tabular}
\end{table}

For image classification with CIFAR-ResNet-8, we run the simple baseline of using $k$-nearest neighbor to directly memorize the labels 
.
We observe the performance: we observe that pure DeepNet has accuracy 56.46\%; pure $k$NN memorization has accuracy 54.44\%; and ResMem has accuracy 59.66\%.


\end{document}